\documentclass[sigconf]{acmart}



\AtBeginDocument{%
  \providecommand\BibTeX{{%
    \normalfont B\kern-0.5em{\scshape i\kern-0.25em b}\kern-0.8em\TeX}}}

\copyrightyear{2022}
\acmYear{2022}
\setcopyright{acmlicensed}\acmConference[KDD '22]{Proceedings of the 28th ACM SIGKDD Conference on Knowledge Discovery and Data Mining}{August 14--18, 2022}{Washington, DC, USA}
\acmBooktitle{Proceedings of the 28th ACM SIGKDD Conference on Knowledge Discovery and Data Mining (KDD '22), August 14--18, 2022, Washington, DC, USA}
\acmPrice{15.00}
\acmDOI{10.1145/3534678.3539451}
\acmISBN{978-1-4503-9385-0/22/08}

\usepackage{amsthm}
\usepackage{amsmath}
\usepackage{xcolor}
\usepackage{float}
\usepackage{accents}
\usepackage{algorithm}
\usepackage{algorithmic}
\usepackage{subcaption}
\usepackage{thmtools}
\usepackage{thm-restate}
\usepackage{tikz}
\usetikzlibrary{automata}
\usepackage{diagbox}
\usepackage{thm-restate}
\usepackage{enumitem}

\DeclareMathOperator{\E}{\mathbb{E}}
\newcommand{\floor}[1]{\left\lfloor #1 \right\rfloor}
\newcommand{\ceil}[1]{\left\lceil #1 \right\rceil}

\newtheorem{theorem}{Theorem}

\newtheorem{lemma}{Lemma}

\DeclareMathOperator{\point}{\mathnormal{j}}
\DeclareMathOperator{\Points}{\mathcal{C}}
\DeclareMathOperator{\labs}{\mathcal{L}}
\DeclareMathOperator{\pcolor}{\mathnormal{h}}
\DeclareMathOperator{\Colors}{\mathcal{H}}
\DeclareMathOperator{\colSize}{\mathnormal{R}}
\DeclareMathOperator{\colPoints}{\mathcal{C}^{\mathnormal{h}}}

\DeclareMathOperator{\nfair}{\mathnormal{n}_{\text{atomic}}}
\newcommand{\nfaircol}[1]{\mathnormal{n}^{h_{#1}}_{\text{atomic}}}
\newcommand{\drop}[1]{\mathnormal{drop(#1)}} 
\newcommand{\la}[1]{\mathnormal{\ell}(#1)}
\newcommand{\lanone}{\mathnormal{\ell}}
\newcommand{\lb}[1]{\mathnormal{(LB)}_{#1}}
\newcommand{\ub}[1]{\mathnormal{(UB)}_{#1}}

\newcommand{\cl}[1]{\mathnormal{(CL)}_{#1}}
\newcommand{\cu}[1]{\mathnormal{(CU)}_{#1}}

\DeclareMathOperator{\LCAL}{\mathrm{LCAL}}
\DeclareMathOperator{\LCUL}{\mathrm{LCUL}}
\DeclareMathOperator{\CLP}{\mathrm{CLP}}

\DeclareMathOperator{\pil}{\mathnormal{p^i_L}}
\DeclareMathOperator{\pild}{\mathnormal{P^i_L}}

\DeclareMathOperator{\POF}{\mathrm{PoF}}

\newcommand{\adult}{\textbf{Adult}}

\newcommand{\credit}{\textbf{CreditCard}}
\newcommand{\cens}{\textbf{Census1990}}

\newcommand{\NCRA}{\textbf{NCRA}}
\newcommand{\FC}{\textbf{FC}}
\newcommand{\LFC}{\textbf{LFC}}

\begin{document}

\title{Fair Labeled Clustering}

\author{Seyed A. Esmaeili}
\affiliation{%
  \institution{University of Maryland, College Park} 
  %\email{esmaeili@cs.umd.edu} 
  \country{Maryland, USA \\esmaeili@cs.umd.edu}
}

\author{Sharmila Duppala}
\affiliation{%
  \institution{University of Maryland, College Park} 
  \country{Maryland, USA \\sduppala@cs.umd.edu}
}

\author{John P. Dickerson}
\affiliation{%
  \institution{University of Maryland, College Park}
  \country{Maryland, USA \\johnd@umd.edu}
}

\author{Brian Brubach}
\affiliation{%
  \institution{Wellesley College}
  \country{Massachusetts, USA \\bb100@wellesley.edu}
}

\renewcommand{\shortauthors}{Seyed A. Esmaeili et al.}


\begin{abstract}
The widespread use of machine learning algorithms in settings that directly affect human lives has instigated significant interest in designing variants of these algorithms that are provably fair. Recent work in this direction has produced numerous algorithms for the fundamental problem of clustering under many different notions of fairness. Perhaps the most common family of notions currently studied is group fairness, in which proportional group representation is ensured in every cluster. We extend this direction by considering the downstream application of clustering and how group fairness should be ensured for such a setting. Specifically, we consider a common setting in which a decision-maker runs a clustering algorithm, inspects the center of each cluster, and decides an appropriate outcome (label) for its corresponding cluster. In hiring for example, there could be two outcomes, positive (hire) or negative (reject), and each cluster would be assigned one of these two outcomes. To ensure group fairness in such a setting, we would desire proportional group representation in every label but not necessarily in every cluster as is done in group fair clustering. We provide algorithms for such problems and show that in contrast to their NP-hard counterparts in group fair clustering, they permit efficient solutions. We also consider a well-motivated alternative setting where the decision-maker is free to assign labels to the clusters regardless of the centers’ positions in the metric space. We show that this setting exhibits interesting transitions from computationally hard to easy according to additional constraints on the problem. Moreover, when the constraint parameters take on natural values we show a randomized algorithm for this setting that always achieves an optimal clustering and satisfies the fairness constraints in expectation. Finally, we run experiments on real world datasets that validate the effectiveness of our algorithms.
\end{abstract}



\maketitle

\section{Introduction}
Machine learning applications have seen widespread use across diverse areas from criminal justice to hiring to healthcare. These applications significantly affect human lives and risk contributing to discrimination~\cite{angwin2016machine,obermeyer2019dissecting}. As a result, research has been directed toward the creation of fair machine learning algorithms~\cite{dwork2012fairness}. Much existing work has focused on the supervised setting. However, significant attention has recently been given to clustering---a fundamental problem in unsupervised learning and operations research. 
% The work in this area can be divided into two fairness considerations: \textbf{(1)} group (demographic) fairness
While many important notions of fair clustering have been proposed, the most relevant to our work is group (demographic) fairness \cite{chierichetti2017fair,ahmadian2019clustering,bercea2018cost,bera2019fair,backurs2019scalable,huang2019coresets,esmaeili2020probabilistic,kleindessner2019fair,davidson2020making}. In many of those works, fairness is maintained at the cluster level by imposing constraints on the proportions of groups present in each cluster. For example, we may require the racial demographics of each cluster to be close to the dataset as a whole (demographic/statistical parity) or that no group is over-represented in any cluster. 
% and more recently \textbf{(2)} individual fairness \cite{jung2019center,kleindessner2020notion,brubach2020pairwise,harris2018approximation,anderson2020distributional}. \SE{not sure that it was important to mention individual fairness.}

While constraining the demographics of each cluster is appropriate in some settings, it may be unnecessary or impractical in others. 
% Specifically, clustering finds many applications in decision making $\citeref$ in addition to data exploration and visualization. 
In decision making applications, each cluster eventually has a specific label (outcome) associated with it which may be more positive or negative than others. If the same label is applied to multiple clusters, we may only wish to bound the demographics of points associated with a given label as opposed to bounding the demographics of each cluster. 
% Further, when the metric space is correlated with group membership it may impossible to get meaningful clusters that preserve the demographics of the whole data set. For example, if a person's home address is a feature influencing the metric space, then that space can be correlated with their racial group membership due to housing segregation.

To be more concrete, consider the application of clustering for market segmentation in order to generate better targeted advertising \cite{chen2012data,aggarwal2004method,tan2018introduction,han2011data}. In this setting, 
% the dataset consists of clients, 
we select or engineer features which are informative for targeted advertising and apply clustering (e.g., $k$-means) to the dataset. Then, we analyze the resulting centers (prototypical examples) and make decisions for targeted advertising in the form of recommending specific products or offering certain deals. These products or deals may have different levels of quality, i.e., we may assign labels such as: \emph{mediocre}, \emph{good}, or \emph{excellent} to each cluster based on the quality of its advertisements. For the clusters of a given label (treated as one), it is possible that a certain demographic would be under-represented in the \emph{excellent} label or that another could be over-represented in the \emph{mediocre} label. In fact, the reports in \cite{probublica_facebook,speicher2018potential,datta2018discrimination} indicate that targeted advertising may under-represent certain demographics for some advertisements. An algorithm that ensures each group is represented proportionally in each label could remedy this issue. While applying group fair clustering algorithms would also ensure demographic representation in the clusters and thus the labels, it could come at the price of a higher deformation in the clustering since points would have to be routed to possibly faraway centers just to satisfy the representation proportions. On the other hand, ensuring fair representation across the labels, but not necessarily the centers is less restrictive and likely to cause less deformation to the clustering. 

% Furthermore, it may be desired to have proportional demographic representation only in a specific label such as \emph{excellent} but not necessarily in every label, this less restrictive constraint is clearly easier to achieve but would not be taken advantage of using group fair algorithms as they enforce demographic representation is every cluster.   
Another similar example is clustering for job screening~\cite{deepak2020whither} in which we have a dataset of candidates,\footnote{In some countries, such as India, the number of candidates can be in the millions for government jobs: \url{https://www.bbc.com/news/world-asia-india-43551719}.} and each candidate is represented as a point in a metric space. Clustering could be applied over this set to obtain $k$ many clusters. Then, the center of each cluster is given a more costly examination (e.g., a human carefully screening a job application). Accordingly, the centers would be assigned labels from the set: \emph{hire}, \emph{short-list}, \emph{scrutinize further}, or \emph{reject}.  Naturally, more than one cluster could be assigned the same label. Clearly, the greater concern here is demographic parity across the labels, but not necessarily the individual clusters. Thus, group fair clustering would yield unnecessarily sub-optimal solutions.

While in the above examples the label of the center was decided according to its position in the metric space. One can envision applications in Operations Research where the label assignment of the center is not dependent on its position \cite{shmoys2004facility,xu2008approximation}. Rather, we would have a set of centers (facilities) of different service types (or quality) and we would have a budget for each  service type. Further, to ensure group fairness we would satisfy the demographic representation over the service types offered. In this setting, we would have to choose the labels so as to minimize the clustering cost subject to further constraints such as budget and fair demographic representation.

The above examples illustrate the need for a group fairness definition at the label level when clustering is applied in decision-making settings or when the different centers (facilities) provide different types of services. 
In addition to being sufficient, evaluating fairness at the label level rather than cluster level can also be necessary. When the metric space is correlated with group membership it may be costly, counterproductive, or impossible to get meaningful clusters that each preserve the demographics of the dataset. For example, if the metric space is geographic as in many facility location problems, a person's location can be correlated with their racial group membership due to housing segregation. The same is true in machine learning when common features like location redundantly encode sensitive features such as race. 
In this case, the more strict approach of group fairness in each cluster could cause a large enough degradation in clustering quality that the entity in charge chooses a classical ``unfair'' clustering algorithm instead. 
% (see figure \ref{gf_costs_more}). 
In legal terms, this unfair clustering approach may exhibit \emph{disparate impact}---members of a protected class may be adversely affected without provable intent on the part of the algorithm. However, disparate impact is allowed if the unfair clustering can be justified by \emph{business necessity} (e.g., the fair clustering alternative is too costly)\cite{CivilRights1991}.

% NOTE: The following paragraph could be cut for space if necessary.
% Along similar lines, enforcing demographic parity of clusters may favor the wrong members of a discriminated group in some settings. When screening job applications from members of two cultural groups A and B using an approach which risks favoring group A, the hireable members of B may not be those most similar to the hireable members of A (See Example 1 in Section 3.1 ``Why is statistical parity insufficient?’’ of~\cite{dwork2012fairness}). 
% If such an approach leads to several less qualified members of B being ``short-listed’’, then a company might say they tried to hire diversely, but did not see any qualified candidates from B.

Thus, our work can be seen as a less stringent, less costly, and fundamentally different approach which still satisfies some similar fairness criteria to existing group fair clustering formulations. In addition, the decision-maker may not be concerned with the demographic representation in all labels, but rather only a specific set of label(s) such as \emph{hire} and \emph{short-list}. It may also be desired to enforce different lower and upper representation bounds for different labels. 

\subsection{Our Contributions}
\label{sec:ourcont}
We introduce the problem of fairness in labeled clustering in which group fairness is ensured within the labels as opposed to each cluster. Specifically, we are given a set of centers found by a clustering algorithm, then having found the centers, we have to satisfy group fairness over the labels. We consider two settings: (1)\textbf{ labeled clustering with assigned labels ($\LCAL$)} where the center labels are decided  based on their position as would be expected in machine learning applications and (2)\textbf{ labeled clustering with unassigned labels ($\LCUL$)} where we are free to select the center labels subject to some constraints. We note that throughout we consider the set of centers to be given and fixed (although in the unassigned setting their labels are unknown), therefore the problem is essentially a routing (assignment) problem where points are assigned to centers rather than a clustering problem. We however, refer to it as clustering since we minimize the clustering cost throughout and since our motivation is clustering based. Moreover, many of the application cases of the assigned labels setting would not alter the centers as that would not change the assigned labels which are given manually through further inspection \cite{deepak2020whither,chen2012data,tan2018introduction} or in the case of the unassigned labels we would have a fixed set of centers. Further, the work of \cite{davidson2020making} in fair clustering follows a similar setting where the centers are fixed. 

For the $\LCAL$ (assigned labels) setting, we show that if the number of labels is constant, then we can obtain an optimal clustering cost subject to satisfying fairness within labels in polynomial time. This is in contrast to the equivalent \emph{fair assignment} problem in fair clustering which is NP-hard\cite{bercea2018cost,esmaeili2021fair}.\footnote{In this equivalent problem, the set of centers is given. We seek an assignment of points to these centers that minimizes a clustering objective and bounds the group proportions assigned to each center.} Furthermore, for the important special case of two labels, we obtain a faster algorithm with running time $O(n(\log{n}+k))$. 

For the $\LCUL$ (unassigned labels) setting, we give a detailed characterization of the hardness under different constraints and show that the problem could be NP-hard or solvable in polynomial time. Furthermore, for a natural specific form of constraints we show a randomized algorithm that always achieves an optimal clustering and satisfies the fairness constraints in expectation.

We conduct experiments on real world datasets that show the effectiveness of our algorithms. In particular, we show that our algorithms provide fairness at a lower cost than fair clustering and that they indeed scale to large datasets. We note that due to the space limit, some proofs are relegated to the appendix.

\section{Related Work}

Much of the investigation into fairness in machine learning and automated systems was sparked by the seminal work of~\cite{dwork2012fairness}. That work and others~\cite{zemel2013reps,Feldman2015Certifying} respond to the reality that points which should receive similar classifications, but belong to different demographic groups may not be near each other in the feature space. Our approach accounts for this phenomenon as well by allowing points from different groups to be distant in the metric space and assigned to different clusters, but receive the same label.

The most closely related work in the clustering space addresses group (demographic) fairness among the members of each cluster~\cite{chierichetti2017fair,backurs2019scalable,bercea2018cost,bera2019fair,ahmadian2019clustering,huang2019coresets,esmaeili2020probabilistic, davidson2020making,esmaeili2021fair}. However, as noted earlier, these approaches can diverge quite a bit from the problem we consider and are not directly comparable. Some work also considers the less related fair data summarization problem of bounding group proportions among the set of centers/exemplars~\cite{kleindessner2019fair}. In addition, several other notions of fair clustering and summarization exist to capture the diverse settings and objectives for which fairness is desirable. These include service guarantees bounding the distance of points to centers~\cite{harris2018approximation}, preserving nearby pairs or communities of points in the metric space~\cite{brubach2020pairwise}, equitable
group representation \cite{abbasi2020fair,ghadiri2021socially}, and fair candidate selection\cite{bei2020candidate}.

In particular, the setting of \cite{davidson2020making} is very similar to ours in that the set of centers is fixed, and the problem amounts to routing points to centers so as to minimize the clustering cost function. However, unlike our work, the constraint is to satisfy conventional group fairness in the clusters; whereas in our setting, we are concerned with  group fairness only within the labels. 
% Although both works are concerned with clustering, the set of centers are fixed and both papers solve an allocation problem instead of a clustering problem.    

% NOTE; The paragraph below could be omitted to save space if needed.
% Prior work has explored many algorithmic facets of bias in hiring problems like the motivating example in the previous section~\cite{raghavan2020mitigating}. These perspectives include job advertising~\cite{datta2015automated}, resume screening~\cite{cowgill2018biasresume}, clustering clients to match to contractors~\cite{verroios2015client}, algorithmic management decisions~\cite{cappelli2019artificial}, and even stereotypes in image search results for occupations~\cite{kay2015unequal}.

% NOTE: These two citations involve clustering transaction data, but I'm not sure how to work them in.
% \cite{ungar1998colabfilter}
% \cite{guidotti2017clustering}

% NOTE: Not sure if we need to bring up equal opportunity again here.
% \cite{hardt2016equality}

% \SE{Giving the traditional revision of clustering may not be useful here (especially that this isn't really doing clustering). Instead I would suggest: 
% (1)going over some results in fair clustering in a high level manner. 
% (2) applications of clustering that fit the problem (having labels.) 
% (3) pointing out how fair clustering is not directly applicable.
% }

% \SE{should explicitly mention somewhere that we should minimize a clustering objective since that means that points are being assigned to nearby centers and that nearby means more similar (compatible)}. 
\section{Preliminaries and Problem Formulation} 
We are given a complete metric graph with a set of vertices (points) $\Points$ where $|\Points|=n$. Further, each point has a color assigned to it according to the function $\chi: \Points \rightarrow \Colors$ where $\Colors$ is the set of possible colors, with cardinality $\colSize$, i.e. $|\Colors|=\colSize$. We refer to the set of points with color $\pcolor \in \Colors$ by $\colPoints$. We further have a distance function $d:\Points \times \Points \rightarrow \mathbb{R}_{\ge 0}$ which defines a metric. We are given a set $S$ of centers that have been selected, $S$ contains at most $k$ many centers, i.e. $|S|\leq k$. Furthermore, we have the set of labels $\labs$ where $\labs$ has a total of $m$ many possible labels, i.e. $|\labs|=m$. The function $\lanone:S \rightarrow \labs$ assigns centers to labels. Our problem always involves finding an assignment from points to centers, $\phi: \Points \rightarrow S$ such that it is the optimal solution to a constrained optimization problem where the objective is a clustering objective. Specifically, we always have to minimize the objectives:$\Big( \sum_{\point \in \Points} d^p(\point,\phi(\point)) \Big)^{1/p}$, where $p=\infty,1,$ and $2$ for the $k$-center, $k$-median, and $k$-means objectives, respectively. We note that for the $k$-center with $p=\infty$, the objective reduces to a simpler form $\Big( \sum_{\point \in \Points} d^p(\point,\phi(\point)) \Big)^{1/p} = \max_{\point \in \Points} d(\point,\phi(\point))$ which is the maximum distance between a point $\point$ and its assigned center $\phi(\point)$. We consider the number of colors $\colSize$ to be a constant throughout. This is justified by the fact that in most applications demographic groups tend to be limited in number. 
% Further, the number of labels is considered to be constant since a number with dependence on the problem size $n$ or $k$ does not seem reasonable and we expect that a decision maker only has a limited number of labels (choices) to decide from.  

As mentioned earlier, we have two settings and accordingly two variants of this optimization: (1) labeled clustering with assigned labels ($\LCAL$) where the centers have already been assigned labels and (2) labeled clustering with unassigned labels ($\LCUL$) where the centers have not been assigned any labels and can be assigned any arbitrary labels from the set $\labs$ subject to (possible) additional constraints.  

We pay special attention to the two label case where $\labs=\{P,N\}$ with $P$ being a positive outcome label and $N$ being a negative outcome label, although many of our results can be extended to the general case where $|\labs|=m>2$.

% The objective of this optimization problem is always the same but the additional constraints depend on whether the centers have known assigned labels or not. More Precisely, the optimization objective is the clustering objective: $\min_{\phi} \Big( \sum_{\point \in \Points} d^p(\point,\phi(\point))\Big)^{1/p}$ where $p=\infty,1, $ or $2$ depending on whether we are doing $k$-center, $k$-median, or $k$-means clustering, respectively. As for the constraints, we 

% Although we always minimize the above mentioned clustering objective, the constraints of the optimization problem are dependent on the settings we have. We define two variants of this optimization: labeled clustering with known labels ($\LCNL$) where the centers have already been assigned labels and labeled clustering with unknown labels ($\LCUL$) where the centers can be assigned any arbitrary labels from the set $\labs$, subject to possible additional constraints on the problem. 
\subsection{Labeled Clustering with Assigned Labels ($\LCAL$):}
In this problem the labels of the centers have been assigned, i.e. the function $\lanone$ is fully known and fixed. We look for an assignment $\phi$ which is the optimal solution to the following problem: 
%\SE{should fix the margins here for (1b)}
{\small
\begin{subequations}\label{lcal_opt}
 \begin{equation}
    \label{LC-11}
    \min_{\phi}  \Big( \sum_{\point \in \Points} d^p(\point,\phi(\point)) \Big)^{1/p}  \\
 \end{equation}
 %\vspace{-4.5mm}
 \begin{align}
    \label{lcal_1}
     & \forall L \in \labs, \forall h \in \Colors: l^{L}_{\pcolor}\text{\hspace{-2mm}} \sum_{\substack{i\in S\\ \la{i}=L}}\text{\hspace{-2mm}} |\Points_i|  \leq \text{\hspace{-3mm}}\sum_{\substack{i\in S\\ \la{i}=L}}\text{\hspace{-2mm}} |\Points^{\pcolor}_i| \leq u^{L}_{\pcolor} \text{\hspace{-2mm}}\sum_{\substack{i\in S\\ \la{i}=L}}\text{\hspace{-2mm}} |\Points_i|
 \end{align}
  %\vspace{-6mm}
  \begin{align}
    \label{lcal_2}
     & \forall L \in \labs:  \lb{L} \leq \sum_{i\in S: \la{i}=L} |\Points_i| \leq \ub{L}
 \end{align}
\end{subequations}
}
\vspace{-0.5mm}
where $\Points_i$ refers to the points $\phi$ assigns to the center $i$, i.e. $\Points_i=\{ \point \in \Points| \phi(\point)=i \}$. $\Points^{\pcolor}_i=\Points_i \cap \colPoints$, i.e. the subset of $\Points_i$ with color $\pcolor$. $l^{L}_{\pcolor}$ and $u^{L}_{\pcolor}$ are lower and upper proportional bounds for color $\pcolor$. Clearly, $l^{L}_{\pcolor},u^{L}_{\pcolor} \in [0,1]$. Constraints (\ref{lcal_1}) are the proportionality (fairness) constraints that are to be satisfied in fair labeled clustering. Notice how we have a superscript $L$ in $l^{L}_{\pcolor}$ and $u^{L}_{\pcolor}$, this is to indicate that we may desire different proportional representations in different labels. For example, for the case of two labels $\labs=\{P,N\}$, we may not want to enforce proportional representation in the negative label so we set $l^{N}_{\pcolor}=0$ and $u^{N}_{\pcolor}=1$ but we may want to enforce lower representation bounds in the positive label and therefore set $l^{P}_{\pcolor}$ to some non-trivial value. Note that these constraints generalize those of fair clustering, in fact we can obtain the constraints of fair clustering by letting each center have its own label ($m=k$) and enforcing the proportional representation bounds to be the same throughout all labels. However, in our problem we focus on the case where the number of labels $m$ is constant since in most applications we expect a small number of labels (outcomes). In fact, a large number could cause a problem in terms of decision making and result interpretability. 

In constraints (\ref{lcal_2}), $\lb{L}$ and $\ub{L}$ are pre-set upper and lower bounds on the number of points assigned to a given label, clearly $\lb{L}, \ub{L} \in \{0,1,\dots,n\}$ . They are additional constraints we introduce to the problem that have not been previously considered in fair clustering. Our motivation comes from the fact that since positive or negative outcomes could be associated with different labels, it is reasonable to set an upper bound on the total number of points assigned to a positive label, since a positive assignment may incur a cost and there is a bound on the budget. Similarly, we may set a lower bound to avoid trivial solutions where most points are assigned to negative outcomes and no or very few agents enjoy the positive outcome. 
% Similar considerations could be made for negative outcomes as well. Note that these constraints (\ref{lcal_2}) generalize the problem and simply setting $\lb{L}=0$ and $\ub{L}=|\Points|=n, \forall L \in \labs$ allows solutions to assign an arbitrary number of points to any label.  

% Looking at problem (\ref{lcal_opt}) it is clear that by setting the constraints to various values we may obtain many meaningful problems. For example, we may even completely ignore the proportional representation constraints (\ref{lcal_1}) (by setting $l^{L}_{\pcolor}=0$ and $u^{L}_{\pcolor}=1, \forall L \in \labs$) and may only have the upper and lower bound constraints of (\ref{lcal_2}). This corresponds to the case where we care about the number of people receiving the outcome but have no fairness considerations. 

\subsection{Labeled Clustering with Unassigned Labels ($\LCUL$):}
In labeled clustering with unassigned labels $\LCUL$, the labels of the centers have not been assigned. As noted, this captures certain OR applications in which the label of a center is not related to its position in the metric space. 
%%% removed the paragraph below
% For example, suppose we want to allocate $k$ mobile health clinics and can install different kinds of services (possibly different levels of quality) in each of them. It is reasonable to desire proportional representation in the number of points receiving a given kind of service (label). Another application would be choosing a subset of polling places to receive the resources to open for early voting in a way that avoids disenfranchising protected groups. More generally, this captures a kind of facility location problem in which some facilities are quantifiably better than others, and we wish to ensure fair access to these unequal resources.

Similar to the case with assigned labels $\LCAL$, we would also wish to minimize the clustering objective. In general we have the following optimization problem:
{\small
\begin{subequations}\label{lcul_opt}
 \begin{equation}
    \label{LC-11}
    \min_{\phi, \lanone}  \Big( \sum_{\point \in \Points} d^p(\point,\phi(\point)) \Big)^{1/p}  \\
 \end{equation}
 \vspace{-2mm}
 \begin{align}
    \label{lcul_1}
       & \forall L \text{\hspace{-1mm}}  \in \labs, \forall h \in \Colors: l^{L}_{\pcolor} \text{\hspace{-2mm}}  \sum_{\substack{i\in S\\ \la{i}=L}} \text{\hspace{-2mm}} |\Points_i|  \leq \text{\hspace{-3mm}}\sum_{\substack{i\in S\\ \la{i}=L}}\text{\hspace{-2mm}} |\Points^{\pcolor}_i| \leq u^{L}_{\pcolor} \text{\hspace{-2mm}} \sum_{\substack{i\in S\\ \la{i}=L}} \text{\hspace{-2mm}} |\Points_i|
 \end{align}
  \vspace{-2mm}
  \begin{align}
    \label{lcul_2}
     & \forall L \in \labs:  \lb{L} \leq \sum_{i\in S: \la{i}=L} |\Points_i| \leq \ub{L}
 \end{align}
  \vspace{-2mm}
   \begin{align}
    \label{lcul_3}
     & \forall L \in \labs: \cl{L}\leq |S^{L}| \leq \cu{L}
 \end{align}
\end{subequations}
}
\vspace{-0.5mm}
Note how in the above objective $\lanone$ has been added as an optimization variable unlike the objective in (\ref{lcal_opt}) for $\LCAL$. Further, we have added constraint (\ref{lcul_3}) where $S^{L}$ refers to the subset of centers that have been assigned label $L$ by the function $\lanone$, i.e. $S^L=\{i \in S|\la{i}=L\}$. This constraint simply lower bounds $S^L$ by $\cl{L}$ and upper bounds it by $\cu{L}$. This constraint models minimal service guarantees (lower bound) and budget (upper bound) guarantees. Clearly, $\cl{L},\cu{L} \in \{0,1,\dots,k\}$. Further, setting $\cl{L}=0$ and $\cu{L}=k  \ \forall L \in \labs$ allows any label to have any number of centers, effectively nullifying the constraint. We show in a subsequent section that forcing certain constraints on the problem can make it NP-hard and that relaxing some constraints would make the problem permit polynomial time solutions. 

\section{Algorithms and Theoretical Guarantees for $\LCAL$}\label{sec:lcal}
\subsection{$\LCAL$ is Polynomial Time Solvable:}\label{sec:lcal_all}
$LCAL$ is problem (\ref{lcal_opt}) where we have a collection of centers and we wish to minimize a clustering objective subject to proportionality constraints (\ref{lcal_1}) and possible constraints on the number of points each label is assigned (\ref{lcal_2}). Fair assignment\footnote{Fair assignment \cite{bercea2018cost,bera2019fair,esmaeili2020probabilistic} is a sub-problem solved in fair clustering to finally yield a full algorithm for fair clustering.} is a problem which has a very similar form to our problem; the centers have already been decided and we wish to satisfy the same proportionality constraints in every cluster, specifically the optimization problem is: 
{\small
\begin{subequations}\label{fa_opt}
 \begin{equation}
    \label{LC-11}
    \min_{\phi}  \Big( \sum_{\point \in \Points} d^p(\point,\phi(\point)) \Big)^{1/p}  \\
 \end{equation}
 \vspace{-4mm}
 \begin{align}
    \label{fa_1}
     & \forall i \in S, \forall h \in \Colors: l_{\pcolor}  |\Points_i|  \leq  |\Points^{\pcolor}_i| \leq u_{\pcolor} |\Points_i|
 \end{align}
\end{subequations}}
It may be thought that the above optimization is simpler than that of $LCAL$ (\ref{lcal_opt}), since all clusters have to satisfy the same proportionality bounds and there is no bound on the total number of points assigned to a any specific cluster. However, \cite{bercea2018cost,esmaeili2021fair} show that the problem is in fact NP-hard for all clustering objectives. We show in the theorem below that $LCAL$ can be solved in polynomial time for all clustering objectives.

\begin{restatable}{theorem}{thlcalpoly}\label{th_lcal_poly}
Labeled clustering with assigned labels $LCAL$ is solvable in polynomial time for the all clustering objectives ($k$-center, $k$-median, and $k$-means). 
\end{restatable}
\begin{proof}
The key observation is that any assignment function $\phi$, will assign a specific number of points $n_{L}$ to the centers with label $L$. Further, we have that $\sum_{L\in \labs} n_L=n$ since all points must be covered. Now, since $|\labs|=m$ is a constant, this means that there is a polynomial number of ways to vary the total number of points distributed across the labels. More specifically, the total number of ways to distribute points across the given labels is upper bounded by $\underbrace{n \times n \times \dots \times n}_{m-1} = n^{m-1}$. Note that once we decide the number of points assigned to the first $(m-1)$ labels, the last label must be assigned the remaining amount to cover all $n$ points, so we have a total of $n^{m-1}$ possibilities. Since we have established, that there is a polynomial number of possibilities for distributing the number of points across the labels, if we can solve $LCAL$ optimally for each possibility and simply take the minimum across all possibilities then we would obtain the optimal solution. 

Now that we are given a specific distribution of number of points across labels, i.e. $(n_1,\dots,n_L,\dots,n_m)$ where $\sum_{L\in \labs} n_L=n$, we have to solve $LCAL$ optimally for that distribution. The problem amounts to routing points to appropriate centers such that we minimize the clustering objective and satisfy the distribution of number of points across the labels along with the color proportionality. To do that we construct a network flow graph and solve the resulting minimum cost max flow problem. The network flow graph is constructed as follows: 
\begin{itemize}[leftmargin=*]
    \item \textbf{Vertices:} the set of vertices is $V=\{s\} \cup \Points \cup (\cup_{\pcolor \in \Colors} S^{\pcolor}) \cup (\cup_{\pcolor \in \Colors} \labs^{\pcolor}) \cup \labs \cup \{t\}$. Vertex $s$ is the source, further we have a vertex for each point, hence the set of vertices $\Points$. For each color $\pcolor \in \Colors$ we create a vertex for each center in $S$ and for each label in $\labs$, these vertices constitute the sets $\cup_{\pcolor \in \Colors} S^{\pcolor}$ and $\cup_{\pcolor \in \Colors}\labs^{\pcolor}$, respectively. We also have a vertex for each label in $\labs$ and finally the sink $t$. 
    \item \textbf{Edges:} the set of edges is $E=E_{s\rightarrow \Points} \cup E_{\Points \rightarrow S^{\pcolor}} \cup E_{S^{\pcolor} \rightarrow \labs^{\pcolor}} \cup E_{\labs^{\pcolor} \rightarrow \labs} \cup E_{\labs \rightarrow t}$. $E_{s\rightarrow \Points}$ consists of edges from the source $s$ to every point $\point \in \Points$, $E_{\Points \rightarrow S^{\pcolor}}$ consists of edges from every point $\point \in \Points$ to the center of vertices of the same color in $S^{\pcolor}$, $E_{S^{\pcolor} \rightarrow \labs^{\pcolor}}$ consists of edges from the colored centers to their corresponding label of the same color, $E_{\labs^{\pcolor} \rightarrow \labs}$ consists of edges from the colored labels to their corresponding label, finally $E_{\labs \rightarrow t}$ consists of edges from every label in $\labs$ to the sink $t$.  
    \item \textbf{Capacities:} the edges of $E_{s\rightarrow \Points}$ have a capacity of 1, the edges of $E_{\labs^{\pcolor} \rightarrow \labs}$ have a capacity of $\floor{u^{L}_{h}n_L}$, the edges of $E_{\labs \rightarrow t}$ have a capacity of $n_L$.
    \item \textbf{Demands:} the vertices of $\labs^{\pcolor}$ have a demand of $\ceil{l^{L}_{\pcolor}n_L}$, the vertices of $\labs$ have a demand of $n_L$. 
    \item \textbf{Costs:} all edges have a cost of zero except the edges of $E_{\Points \rightarrow S^{\pcolor}}$ where the cost of the edge between the point and the center is set according to the distance and the clustering objective ($k$-median or $k$-means). As noted earlier a vertex $\point$ will only be connected to the same color vertex that represents center $i$ in the network flow graph, we refer to that vertex by $i^{\chi(\point)}$ and clearly $i^{\chi(\point)} \in S^{\chi(\point)}$. Specifically, $\forall (\point,i^{\chi({\point})}) \in E_{\Points \rightarrow S^{\pcolor}},  \text{cost}(\point,i^{\chi({\point})})=d^p(\point,i)$ where $p=1$ for the $k$-median and $p=2$ for the $k$-means.  
\end{itemize}
%figure \ref{nf_lc_fig} shows an example of this flow problem. 
We write the cost for a constructed flow graph as $\sum_{\point \in \Points, i \in S} d^p(\point,i) x_{ij}$ where $x_{ij}$ is the amount of flow between vertex $\point$ and center $i^{\chi({\point})}$. Since all capacities, demands, and costs are set to integer values. Therefore we can obtain an optimal solution (maximum flow at a minimum cost) in polynomial time where all flow values are integers. Therefore, we can solve $\LCAL$ optimally for a given distribution of points.  

The above construction are for the $k$-median and $k$-means. For the $k$-center we slightly modify the graph. First, we point out that unlike the $k$-median and $k$-means, for the $k$-center the objective value has only a polynomial set of possibilities ($kn$ many exactly) since it is the distance between a center and a vertex. So our network flow diagram is identical but instead of setting a cost value for the edges in edges of $E_{\Points \rightarrow S^{\pcolor}}$, we instead pick a value $d$ from the set of possible distances $d(\point,i) \text{ where } \point \in \Points, i \in S$ and draw an edge between a point $\point$ and a center $i^{\chi(\point)}$ only if $d(\point,i)\leq d$. Also we do not need to solve the minimum cost max flow problem, instead the max flow problem is sufficient.
\end{proof}

%\begin{theorem}\label{th_lcal_poly}
%Labeled clustering with assigned labels $LCAL$ is solvable in polynomial time for the all clustering objectives ($k$-center, $k$-median, and $k$-means). 
%\end{theorem}

% \begin{figure}[h!]
% \centering
% \hspace*{-1cm} % Moves the figure slightly left
% \includegraphics[width=1.2\linewidth]{Figs/diag-hardness1.png}
% \caption{Network flow diagram to solve $\LCAL$ in polynomial time.}
% \label{nf_lc_fig}
% \end{figure}
% \end{proof}

\subsection{Efficient Algorithms for $\LCAL$ for the Two Label Case:}\label{sec:lcal_fast_two_label_alg}
For the $k$-median and $k$-means and the two label case we present an algorithm with $O(n(\log{(n)}+k))$ running-time. The intuition behind our algorithm is best understood for the  case with ``exact population proportions'' for both the positive and negative labels\footnote{The general  case is shown in the appendix.}. First, we note that each color $\pcolor \in \Colors$ exists in proportion $r_{\pcolor}=\frac{|\colPoints|}{|\Points|}$ where we refer to $r_{\pcolor}$ as the population proportion. The case of exact population proportions for the positive and negative labels, is the one where $\forall \pcolor \in \Colors, \forall L \in \{P,N\}: l^{L}_{\pcolor}=u^{L}_{\pcolor}=r_{\pcolor}=\frac{|\colPoints|}{|\Points|}$

That is, the upper and lower proportion bounds coincide and are equal to the proportion of the color in the entire set. This forces only a limited set of possibilities for the total number of points (and their colors) which we can assign to either $P$ or $N$. For example, if we have two colors and $r_1=r_2=\frac{1}{2}$, then we can only assign an equal number of red and blue points to $P$ and likewise to $N$. For the case of three colors with $r_1=\frac{1}{3}, r_2=\frac{1}{2}, r_3=\frac{1}{6}$, then we can only assign points of the following form across the different labels: $\text{points for the first color} = 2c, \text{points for the second color} = 3c, \text{points for the third color} = c$ where $c$ is a non-negative integer. We refer to this smallest "atomic" number of points by $\nfair$ and the number of color $h$ of its subset by $\nfaircol{}$. 
% The following fact formalizes this: 
% \begin{fact}
% For exact preservation the number of points that can be assigned to either the positive or negative set across the colors $h_1,h_2,\dots,h_{\colSize}$  is an integer multiple of $\nfaircol{1},\nfaircol{2},\dots, \nfaircol{\colSize}$. We define the sum of these number of points by $\nfair = \sum_{i=1}^R \nfaircol{i}$. Note that it also follows that $n=s\nfair$ where $s$ is a positive integer. 
% \end{fact}

% We note that by simply looking for the smallest integer $m'$ such that $\forall \pcolor \in \Colors: r_{\pcolor}m' \text{ is an integer}$ we can find $\nfaircol{}$ and therefore $\nfair$ as well. This can clearly be done in $O(n)$ time. 

Now we define some notation $P(\point)=\min_{i\in P} d(\point,i)$ and $N(\point)=\min_{i\in N} d(\point,i)$, i.e. the distance of the closest centers to $\point$ in $P$ and $N$, respectively. Further, $\phi^{-1}(P)$ and  $\phi^{-1}(N)$ are the set of points assigned to the positive and negative centers by the assignment $\phi$, respectively. We can now define the drop of a point $\point$ as $\drop{\point}=N(\point)-P(\point)$, clearly the larger $\drop{\point}$ the higher the cost goes down as we move it from the negative to the positive set. We can obtain a sorted values of $drop$ for each color in $O\big(n(\log{n}+k)\big)$ run-time.

The algorithm is shown (algorithm block (\ref{alg:alg_exact_k_med_means})). In the first step we start with all points in $N$, then in step 2 we move the minimum number of $\nfaircol{}$ for each color $h$ to satisfy the size bounds for each label (constraint (\ref{lcal_2})). Finally in the loop starting at step 3, we move more points to the positive label (in an ``atomic'' manner) if it lowers the cost and is within the size bounds.
% The algorithm involves essentially two steps (although the second step is a loop). In the first step no point is assigned to a center in $P$ and instead all are assigned to $N$. This can be done by solving a min-cost max-flow instance. However, the following solution is faster, simply assign each point $\point$ to its closest center in $N$. The optimality of this assignment is shown in the base case proof of theorem \ref{th_exact_alg_correct}. In the second step, we loop over the possible number of assignments for $P$ by finding an assignment $\phi$ with  $|\phi^{-1}(P)|$ incremented by $\nfair$ in each iteration and moving for each color $\pcolor$, $\nfaircol{}$ many points from the negative set of centers to the positive set. This set of points which is moved to the positive set is selected greedily, i.e. it is the set which has the maximum drop to the cost. The drop of a point $\point$ is defined as $\drop{\point}=N(\point)-P(\point)$, clearly the larger $\drop{\point}$ the higher the cost goes down as we move it from the negative to the positive set. The optimality of step 2 is shown in the inductive step of the proof of theorem \ref{th_exact_alg_correct}.   

\begin{algorithm}[h!]
   \caption{Exact Preservation for $k$-median / $k$-means }
   \label{alg:alg_exact_k_med_means}
\begin{algorithmic}[1]
   %\STATE Define the optimal assignment as $\phi^{*}$ and its cost as $cost^{*}$.
   %\STATE 
   \STATE Find an assignment $\phi_0$ that assigns all points to their nearest center in $N$, this means that $|\phi^{-1}_{0}(N)|=n$ and $|\phi^{-1}_{0}(P)|=0$. Set $\phi^{*}=\phi_0$. 
   %\STATE \textbf{Step 2:} 
   \STATE Move $q_h=r_h \max\{\lb{P},n-\ub{N}\}$ many points of color $h$ with the highest values in $drop$ from the negative label to the positive label 
   \FOR{$i=\big(\frac{n}{\sum_{h \in \Colors} q_h}\big)$ to $\frac{n}{\nfair}$ }  
   \STATE Take $\nfaircol{}$ many points from each color $h$ with the highest values in $drop$, call the new assignment $\phi'$.
   \IF{${\phi'}^{-1}(P)$ and  ${\phi'}^{-1}(N)$ are within bounds \textbf{and} $cost(\phi')<cost(\phi^*)$}
   \STATE update the assignment to $\phi^*=\phi'$
   \ELSE   
   \STATE break 
   \ENDIF 
   \ENDFOR  
\end{algorithmic}
\end{algorithm}
\vspace{-0.1cm}
\begin{restatable}{theorem}{thmexactalgcorrectrun} \label{th_exact_alg_correct}
Algorithm (\ref{alg:alg_exact_k_med_means}) finds the optimal solution and runs in $O\Big(n(\log{n}+k)\Big)$ time.
\end{restatable}
\begin{proof}
First we prove that the solution is feasible. Constraint (\ref{lcal_1}) for the color proportionality holds, this can is clearly the case before the start of the loop since the centers with negative labels cover the entire set which is color proportional and the the centers with positive labels cover cover nothing which is also color proportional. In each iteration, we move an atomic number of each color from the negative to the positive label and hence both the negative and the positive set of centers satisfy color proportionality in the points they cover.  

For constraint (\ref{lcal_1}) because of exact preservation of the color proportions, we can always tighten the bounds $\lb{L}$ and $\ub{L}$ for each label $L$ such that there multiples of $\nfair$ without modification to the problem, so we assume that $\lb{N}=a\nfair,\lb{P}=b\nfair,\ub{N}=a'\nfair,\ub{P}=b'\nfair$ where $a,a',b,b'$ are non-negative integers and clearly $a\leq b$ and $a' \leq b'$. Step 2 satisfies the lower bound on the number of points in the positive label and the upper bound for the negative set. Note that if this step fails then the problem has infeasible constraints. Further, since we have moved the minimum number of points from the negative set to the positive set, it follows that the upper bounds on the positive are also satisfied since $\lb{P}\leq \ub{P}$, also the lower bound on the negative set is also satisfied since $\lb{N} \leq \ub{N}$. Finally in step 5, the size bounds are always checked fair therefore both labels are balanced. 

Optimally follows since we move the points with the highest $drop$ value to the positive set (these are also the points closest to the positive set). Further, in step 5 we stop moving any points to the positive if there isn't a reduction in the clustering cost. Note that since the values in $drop$ are sorted, another iteration would not reduce the cost. 

Finding the closest center of each label for every point takes $O(nk)$ time. Finding and sorting the values in $drop$ clearly takes $O(n\log{n})$ time. The algorithm does constant work in each iteration for at most $n$ many iterations. Thus, the run time is $O\big(n(\log{n}+k)\big)$.
\end{proof}
% \SD{see appendix}
% \begin{restatable}{theorem}{thmexactalgruntime}\label{th_exact_alg_run_time}
% Algorithm (\ref{alg:alg_exact_k_med_means}) runs in $O\Big(n(\log{n}+k)\Big)$ time. 

% \end{restatable}
% \SD{see appendix}
%\begin{theorem} \label{th_exact_alg_run_time}
%Algorithm (\ref{alg:alg_exact_k_med_means}) runs in $O\Big(n(\log{n}+k)\Big)$ time. 

%\end{theorem}
% \begin{proof}. %%%MARK
% For step 1, finding the closest center in $N$ to each point takes time $n|N|=O(nk)$. 

% To find points with maximum drop in each iteration we create an array $\textbf{Drop}^{\pcolor}$ for each color $\pcolor$ where $[\textbf{Drop}^{\pcolor}]_{\point}=(\text{index of point $\point$},\allowbreak N(\point)-P(\point))$, this takes time $O(nk)$. After we obtain $\textbf{Drop}^{\pcolor}$ we sort $\textbf{Drop}^{\pcolor}$ according to its drop value and in each iteration we take the $\nfaircol{}$ points from the top of the array and increment the pointer by $\nfaircol{}$. Sorting $\textbf{Drop}^{\pcolor}$ takes time $O(|C^{\pcolor}|\log{|C^{\pcolor}|})=O(n\log{n})$. 

% Each iteration takes $\nfair$ many steps to move points to the positive set from the negative set. Since we have $s=\frac{n}{\nfair}$ many iterations, it follows that we have $n$ many total steps.

% Therefore, the overall run time is $O(n(\log{n}+k))$.
% \end{proof}
\vspace{-0.1cm}
With more elaborate conditional statements, the above algorithms can be generalized to give all solution values for arbitrary choices of label size bounds (constraint(\ref{lcal_2})) with the same asymptotic run-time. Such a solution would be useful as it would enable the decision maker to see the complete trade-off between the label sizes and the clustering cost (quality). 
\section{Algorithms and Theoretical Guarantees for $\LCUL$}\label{sec:lcul}
% \SE{$\LCUL$ is fixed param tractable for a constant number of labels.}
% \\ 
% \SE{$\LCUL$ with only constraint is solvable in polynomial time even for a super-constant number of labels, expect for the constraint on the number of points a lable should have.}
% \\
% \SE{$\LCUL$ is NP-hard if only one constraint is dropped.}
% \\
% \SE{$\LCUL$ should also note that the violation of 1 is not there once the difference between the lower bound and upper center numbers is an integer}
\subsection{Computational Hardness of $\LCUL$}\label{sec:lcul_hardness}
We start by discussing the hardness of $\LCUL$. In contrast to $\LCAL$, the $\LCUL$ problem it not solvable in polynomial time. In the fact, the following theorem shows that even if we were to drop one constraint for the $\LCUL$ (problem (\ref{lcul_opt})) we would still have an NP-hard problem. 

% If we consider the case where the color proportionality constraints (\ref{lcul_1}) are completely ignored, then we show that the remaining constraints are sufficient to make the problem NP-hard as we demonstrate in the following theorem: 
\begin{restatable}{theorem}{thmhardnessone}\label{hardness_1}
For the $\LCUL$ problem with two labels and two colors, dropping one of the constraints(\ref{lcul_1}), (\ref{lcul_2}), or (\ref{lcul_3}) still leads to an NP-hard problem. 
\end{restatable}

% Now if instead we were to remove constraints (\ref{lcul_2}) which constrain the total number of points that belong to each label, then we would also have an NP-hard problem as shown in the following theorem:
% \begin{restatable}{theorem}{thmhardnesstwo}\label{hardness_2}
% Even if we do not specify the number of points a label should receive (constraints (\ref{lcul_2}) are ignored), $\LCUL$ is NP-hard.   
% \end{restatable}

% Similarly, if we were to instead ignore the constraint on the number of centers that belong to a specific label, we would still have an NP-hard problem: 
% \begin{restatable}{theorem}{hardnessthree}\label{hardness_3}
% Even if we do not specify the number of centers of each label (ignoring constraints (\ref{lcul_3})), $\LCUL$ is NP-hard. 
% \end{restatable}

Having established the hardness of $\LCUL$ for different sets of constraints, we show that it is fixed-parameter tractable\footnote{An algorithm is called fixed-parameter tractable if its run-time is $O(f(k)n^c)$ where $f(k)$ can be exponential in $k$, see \cite{cygan2015parameterized} for more details.} for a constant number of labels. This immediately follows since a given choice of labels for the centers leads to an instance of $\LCAL$ which is solvable in polynomial time and there are at most $m^k$ many possible choice labels. 
\begin{restatable}{theorem}{hardnessfour}\label{hardness_4}
The $\LCUL$ problem is fixed-parameter tractable with respect to $k$ for a constant number of labels. 
\end{restatable}

It is also worth wondering if the problem remains hard if we were to drop two constraints and have only one instead. Interestingly, we show that even for the case where the number of labels $m$ is super-constant ($m=\Omega(1)$) , if we only had the color-proportionality constraint (\ref{lcul_1}) or the constraint on the number of labels (\ref{lcul_2}), then the problem is solvable in polynomial time. However, if we only had constraint (\ref{lcul_3}) for the number of centers a label has, the problem is still NP-hard. 
\begin{restatable}{theorem}{hardnessd}\label{hardness_5}
Even if number of labels $m=\Omega(n)$, the $\LCUL$ problem is solvable in polynomial time under constraint (\ref{lcul_1}) alone or constraint (\ref{lcul_3}) alone. However, it is NP-hard under constraint (\ref{lcul_2}) alone. 
\end{restatable}

\subsection{A Randomized Algorithm for label proportional $\LCUL$:}\label{sec:lcul_algs}
Here we consider a natural special case of the $\LCUL$ problem which we call color and label proportional case ($\CLP$) where the constraints are restricted to a specific form. In $\CLP$ each label must have color proportions ``around'' that of the population, i.e. color $h$ has proportion $r_h$ in each label $L \in \labs$. Further, each label has a proportion $\alpha_L \in [0,1]$ and $\sum_{L \in \labs} \alpha_L=1$, this proportion decides the number of points the label covers and the number of centers it has. I.e., label $L$ covers around $\alpha_L n$ many points and has around $\alpha_L k$ many centers. Therefore, the optimization takes on the following form below where we have included the $\epsilon$ values to relax the constraints (note that for every value of $\epsilon$, we have that $\epsilon \ge 0$): 
{\small
\begin{subequations}\label{CLP_opt_label proportiona}
 \begin{equation}
    \label{CLP-11}
    \min_{\phi, \lanone}  \Big( \sum_{\point \in \Points} d^p(\point,\phi(\point)) \Big)^{1/p}  \\
 \end{equation}
 \vspace{-2mm}
 \begin{align}
    \label{CLP_1}
     & \forall  L \text{\hspace{-1mm}} \in \text{\hspace{-1mm}} \labs, \forall h \text{\hspace{-1mm}}\in \text{\hspace{-1mm}} \Colors: (r_h \text{\hspace{-1mm}}-\epsilon^A_{h,L}) \text{\hspace{-2mm}} \sum_{\substack{i\in S:\\ \la{i}=L}}\text{\hspace{-2mm}} |\Points_i|  \leq  \text{\hspace{-3mm}}\sum_{\substack{i\in S:\\ \la{i}=L}}\text{\hspace{-2mm}} |\Points^{\pcolor}_i|  \leq (r_h+{\epsilon'}^{A}_{h,L})\text{\hspace{-2mm}} \sum_{\substack{i\in S:\\ \la{i}=L}} \text{\hspace{-2mm}}|\Points_i|
 \end{align}
  \vspace{-3mm}
  \begin{align}
    \label{CLP_2}
     & \forall L \in \labs:  (\alpha_L-\epsilon^B_{L}) n \leq \sum_{i\in S:\la{i}=L} |\Points_i| \leq (\alpha_L+{\epsilon'}^B_{L}) n
 \end{align}
  \vspace{-4mm}
   \begin{align}
    \label{CLP_3}
     & \forall L \in \labs: (\alpha_L-{\epsilon}^C_{L}) k \leq |S^{L}| \leq (\alpha_L+{\epsilon'}^C_{L}) k
 \end{align}
\end{subequations}
}
\vspace{-0.5cm}

% Here we consider a special natural case of the $\LCUL$ problem. We call this case the color and label proportional $\CLP$ problem. In the $\CLP$ setting the optimization of (\ref{lcul_opt}) has the same objective but the constraints take on ``natural" values. In particular, \textbf{(A)} for each label $L$ and each color $h$ which exists in proportion $r_h=\frac{|\Points^h|}{|\Points|}$ in the dataset, the color proportionality constraint is set to the following: $(r_h-\epsilon^A_{h,L}) \sum_{i\in S:\la{i}=L} |\Points_i|  \leq \sum_{i\in S: \la{i}=L} |\Points^{\pcolor}_i| \leq (r_h+{\epsilon'}^{A}_{h,L}) \sum_{i\in S: \la{i}=L} |\Points_i|$, where ${\epsilon}^{A}_{h,L}$ and ${\epsilon'}^{A}_{h,L}$ are positive. This implies that color proportions equal to the population are allowed, i.e. we do not force color under-representation or over-representation in any label. \textbf{(B:)} each label $L$ has a number $\alpha_L \in [0,1]$ and $\sum_{L\in \labs} \alpha_L=1$ where $\alpha_L$ represents the proportions of the population that belong to label $L$. Specifically, we have the following constraint: $(\alpha_L-\epsilon^B_{L}) n \leq \sum_{i\in S:\la{i}=L} |\Points_i| \leq (\alpha_L+{\epsilon'}^B_{L}) n$. \textbf{(C:)} each label $L$ should also have a number of centers around the proportion of $\alpha_L$, this leads to the following constraint: $(\alpha_L-\epsilon^C_{L}) k \leq |S^{L}| \leq (\alpha_L+{\epsilon'}^C_{L}) k$. This leads to the following optimization: 
 
We note that even when the constraints take on this specific form the problem is still NP-hard as shown in the theorem below:
\begin{restatable}{theorem}{twocolorhardness}\label{twocolorhardness}
The $\CLP$ problem is NP-hard even for the two color and two label case. 
\end{restatable}
%\begin{theorem}
%\end{theorem}

We show a randomized algorithm (algorithm block (\ref{alg:alg_rand_lcul})) which always gives an optimal cost to the clustering and satisfies all constraints in expectation and further satisfies constraint (\ref{CLP_3}) deterministically with a violation of at most 1. Our algorithm is follows three steps. In step 1 we find the assignment $\phi^*$ by assigning each point to its nearest center, thereby guaranteeing an optimal clustering cost. In step 2, we set the center-to-label probabilistic assignments $\pil=\alpha_L$. Then in step 3, we apply dependent rounding, due to~\citet{gandhi2006dependent}, to the probabilistic assignments to find the deterministic assignments. This leads to the following theorem:
%\vspace{-0.4cm}
\begin{restatable}{theorem}{thmrandalg}\label{th_alg_rand_lcul}
Algorithm \ref{alg:alg_rand_lcul} gives an optimal clustering and satisfies constraints (\ref{CLP_1},\ref{CLP_2},\ref{CLP_3}) in expectation with (\ref{CLP_3}) being satisfied deterministically at a violation at most 1.
\end{restatable}
\begin{proof}
The optimality of the clustering cost follows immediately since each point is assigned to its closest center.
Now, we show that the assignment satisfies all of the constraints. We have $\pil=\alpha_L$ for each center $i$. Now we prove that constraints (\ref{lcul_1},\ref{lcul_2},\ref{lcul_3}) hold in expectation over the assignments $\pild$. Note that $\pild$ is also an indicator random variable for center $i$, taking label $L$. Then we can show that using property \textbf{(A)} of dependent rounding (marginal probability) that:
\begin{align*}
    & \E[\sum_{i\in S:\ell(i)=L} |\Points_i|]  =  \E[\sum_{i\in S} |\Points_i| \pild]  = \sum_{i\in S} |\Points_i| \E[  \pild] \\ 
    & = \sum_{i\in S} |\Points_i| \pil   = \alpha_L \sum_{i\in S} |\Points_i| = \alpha_L n 
\end{align*}
Clearly, constraint (\ref{CLP_2}) is satisfied. 
Through a similar argument we can show that the rest of the constraints also hold in expectation.
% $\E[\sum_{i\in S:\ell(i)=L} |\Points^h_i|] = \sum_{i\in S} |\Points^h_i| \pil = r_h \sum_{i\in S} |\Points_i| \pil= \E[\sum_{i\in S:\ell(i)=L} |\Points_i|]$

% where the last equality holds from property \textbf{(A)} of dependent rounding (marginal probability) and the equality before holds by the linearity of the expectation. Following an identical derivation we can also show that $\E[\sum_{i\in S:\ell(i)=L} |\Points^h_i|] = \sum_{i\in S} |\Points^h_i| \pil $. Since $\pil$ satisfies constrains (\ref{rand_lcul_1}) and (\ref{rand_lcul_2}), it follows that the constraints are satisfied in expectation, i.e. we have: $\forall L \in \labs, \forall h \in \Colors: l^{L}_{\pcolor} \E\Big[\sum_{i\in S: \la{i}=L} |\Points_i| \Big] \leq \E\Big[\sum_{i\in S: \la{i}=L} |\Points^{\pcolor}_i| \Big] \leq u^{L}_{\pcolor} \E\Big[ \sum_{i\in S: \la{i}=L} |\Points_i| \Big]$ and that $\forall L \in \labs: \lb{L} \leq \E\Big[ \sum_{i\in S: \la{i}=L} |\Points_i| ] \leq \ub{L}$. 

% For the number of centers of label $L$, we have:
% $\E[|S^L|]=\sum_{\alpha_L}=\alpha_L k$, but we note because of dependent rounding's degree-preservation will allow us to have a violation of at most 1. 

We have that $\forall L \in \labs: |S^L|=\sum_{i \in S} \pild = \sum_{i \in S} \alpha_L = \alpha_L k$. By property \textbf{(B)} of dependent rounding (degree preservation) we have $\forall L \in \labs: |S^L| \in \{ \floor{\alpha_L k}, \ceil{\alpha_L k} \}$. Therefore constraint (\ref{CLP_3}) is satisfied in every run of the algorithm at a violation of at most 1.  
\end{proof}

\vspace{-0.2cm}
\begin{algorithm}[h!]
   \caption{Randomized $\LCUL$ Algorithm}
   \label{alg:alg_rand_lcul}
\begin{algorithmic}[1]
   \STATE  Find the assignment $\phi^*$ by assigning each point to its nearest center in $S$. 
   \STATE  For each center $i$, set its probabilistic assignment for label $L$ to $\pil=\alpha_L$. 
   \STATE  Apply dependent rounding \cite{gandhi2006dependent} to probabilistic assignments $\pil$ to get the deterministic assignments $\pild $
\end{algorithmic}
\end{algorithm}

We note that dependent rounding enjoys the \textbf{Marginal Probability} property which means that $\Pr[\pild=1]=\pil$. This enables us to satisfy the constraints in expectation. While we note that letting each center $i$ take label $L$ with probability $\alpha_L$ would also satisfy the constraints in expectation. Dependent rounding also has the \textbf{Degree Preservation} property which implies that $\forall L \in \labs: \sum_{i\in S} \pild \in \{\floor{\sum_{i\in S} \pil}, \ceil{\sum_{i\in S} \pil}\}$ which leads us to satisfy constraint (\ref{CLP_3}) deterministically (in every run of the algorithm) with a violation of at most 1. Further, dependent rounding has the \textbf{Negative Correlation} property which under some conditions leads to a concentration around the expected value. Although, we cannot theoretically guarantee that we have a concentration around the expected value, we observe empirically (section \ref{exp_lcul}) that dependent rounding is much better concentrated around the expected value, especially for constraint (\ref{CLP_2}) for the number of points in each label.

\section{Experiments}\label{sec:experiments}
We run our algorithms using commodity hardware with our code written in Python 3.6  using the \texttt{NumPy} library and functions from the \texttt{Scikit-learn} library \cite{pedregosa2011scikit}. We evaluate the performance of our algorithms over a collection of datasets from the UCI repository \cite{dua2017uci}. For all datasets, we choose specific attributes for group membership and use numeric attributes as coordinates with the Euclidean distance measure. Through all experiments for a color $\pcolor \in \Colors$ with population proportion $r_{\pcolor}=\frac{|\Points^{\pcolor}|}{|\Points|}$  we set the the upper and lower proportion bounds to $l_{\pcolor}=(1-\delta)r_{\pcolor}$ and $u_{\pcolor}=(1+\delta)r_{\pcolor}$, respectively. Note that the upper and lower proportion bounds are the same for both labels. Further, we have $\delta \in [0,1]$, and smaller values correspond to more stringent constraints. In our experiments, we set $\delta$ to 0.1. For both the $\LCAL$ and $\LCUL$ we measure the price of fairness $\POF=\frac{\text{fair solution cost}}{\text{color-blind solution cost}}$ where $\text{fair solution cost}$ is the cost of the fair variant and $\text{color-blind solution cost}$ is the cost of the ``unfair'' algorithm which would assign each point to its closest center. 

We note that since all constraints are proportionality constraints, we calculate the proportional violation. To be precise, for the color proportionality constraint (\ref{lcul_1}), we consider a label $L$ and define $\Delta^L_h \in [0,1]$ where $\Delta^L_h$ is the smallest relaxation of the constraint for which the constraint is satisfied, i.e. the minimum value for which the following constraint is feasible given the solution: $(l^{L}_{\pcolor}-\Delta^L_h)   \sum_{i\in S:\\ \la{i}=L}  |\Points_i|  \leq \sum_{i\in S:\\ \la{i}=L} |\Points^{\pcolor}_i| \leq (u^{L}_{\pcolor}+\Delta^L_h)  \sum_{i\in S:\\ \la{i}=L}  |\Points_i|$, having found $\Delta^L_h$ we report $\Delta_{\text{color}}$ where $\Delta_{\text{color}}=\max_{\{h \in \Colors, l \in \labs\}} \Delta^L_h$. Similarly, we define the proportional violation for the number of points $\Delta^L_{\text{points/label}}$ assigned to a label as the minimal relaxation of the constraint for it to be satisfied. We set $\Delta_{\text{points/label}}$ to the maximum across the two labels. In a similar manner, we define $\Delta_{\text{center/label}}$ for the number of centers a label receives.

% , whereas the $\text{fair solution cost}$ is the cost of the solution resulting from the fair algorithm we run. 
We use the $k$-means++ algorithm \cite{arthurk} to open a set of $k$ centers. These centers are inspected and assigned a label. Further, this set of centers and its assigned labels are fixed when comparing to baselines other than our algorithm. 

%\vspace{-0.4cm}
\paragraph{Clustering Baseline:} In the labeled setting and in the absence of our algorithm, the only alternative that would result in. a fair outcome is a fair clustering algorithm. Therefore we compare against fair clustering algorithms. The literature in fair clustering is vast, we choose the work of \cite{bera2019fair} as it can be tailored easily to this setting in which the centers are open. Further, it allows both lower and upper proportion bounds in arbitrary metric spaces and results in fair solutions at relatively small values of $\POF$ compared to larger $\POF$ (as high as 7) reported in \cite{chierichetti2017fair}. Our primary concern here is not to compare to all fair clustering work, but gauge the performance of these algorithms in this setting. We also compare against the ``unfair'' solution that would simply assign each point to its closest center which we call the nearest center baseline. Though this in general would violate the fairness constraints it would result in the minimum cost. 

% Unlike the $\LCAL$ setting, for the $\LCUL$ setting the centers don't have assigned labels, we assign label $L$ for each center with probability $\alpha_L$. For the $\LCUL$ setting we also compare against a simple baseline where each point is assigned to its closest center and each center is assigned label $L$ with probability $\alpha_L$. See the appendix for additional experimental results.

% In appendix \ref{app_unfair} we also compare against the optimal solution which is completely agnostic to fairness and we show that the resulting solutions could indeed be unfair having a large deviation from the pre-set demographic representation bounds. 

% \vspace{-0.4cm}
% \paragraph{Hardware and Software:} We use only commodity hardware throughout our experiments: MacBook Pro with 2.3GHz Intel Core i5 processor and 8GB 2133MHz LPDDR3 memory. Our code is in Python 3.6, using the \texttt{NumPy} library \cite{oliphant2006guide} and functions from the \texttt{Scikit-learn} library \cite{pedregosa2011scikit}. The implementation of the group fair clustering \cite{bera2019fair} is available online, and we modify it to fit our setting.
%\cite{harris2020array}

%\vspace{-0.4cm}
\paragraph{Datasets:} 
We use two datasets from the UCI repository: The \adult{} dataset consisting of $32{,}561$ points and the \credit{} dataset consisting of $30{,}000$ points. For the group membership attribute we use race for \adult{} which takes on 5 possible values (5 colors) and  marriage for \credit{} which takes on 4 possible values (4 colors). For the \adult{} dataset we use the numeric entries of the dataset (age, final-weight, education, capital gain, and hours worked per week) as coordinates in the space. Whereas for the \credit{} dataset we use age and 12 other financial entries as coordinates.  

\vspace{-0.2cm}
\subsection{$\LCAL$ Experiments}\label{exp_lcal}
\paragraph{\adult{} \textbf{Dataset}:} After obtaining $k$ centers using the $k$-means++ algorithm, we inspect the resulting centers. In an advertising setting, it is reasonable to think that advertisements for expensive items could be targeting individuals who obtained a high capital gain. Therefore, we choose centers high in the capital gain coordinate to be positive (assign an advertisement for an expensive item). Specifically, centers whose capital gain coordinate is $\ge 1{,}100$ receive a positive label and the remaining centers are assigned a negative one. Such a choice is somewhat arbitrary, but suffices to demonstrate the effectiveness of our algorithm. In real world scenarios, we expect the process to be significantly more elaborate with more representative features available. We run our algorithm for $\LCAL$ as well as the fair clustering algorithm as a baseline. Figure \ref{adult_all_lcal} shows the results. It is clear that our algorithm leads to a much smaller $\POF$ and the $\POF$ is more robust to variations in the number of clusters. In fact, our algorithm can lead to a $\POF$ as small as 1.0059 ($0.59\%$) and very close to the unfair nearest center baseline whereas fair clustering would have a $\POF$ as large as 1.7 ($70\%$). Further, we also see that the unlike the nearest center baseline, fair labeled clustering has no proportional violations just like fair clustering. 
% Thus fair labeled clustering satisfies fairness at a minimum cost.

Here for the $\LCAL$ setting, we compare to the optimal (fairness-agnostic) solution where each point is simply routed to its closest center regardless of color or label. We use the same setting at that from section \ref{sec:experiments}. We set $\delta=0.1$ and measure the $\POF$. Since the (fairness-agnostic) solution does not consider the fairness constraint we also measure its proportional violations. Figures 6 and 7 show the results over the \adult{} and \credit{} datasets. We can clearly see that although the (fairness-agnostic) solution has the smallest cost it has large color violation. We also see that our algorithm unlike fair clustering achieves fairness but at a much lower $\POF$.

\vspace{-.15cm}
\begin{figure}[h!]
\centering
\includegraphics[width=0.6\linewidth]{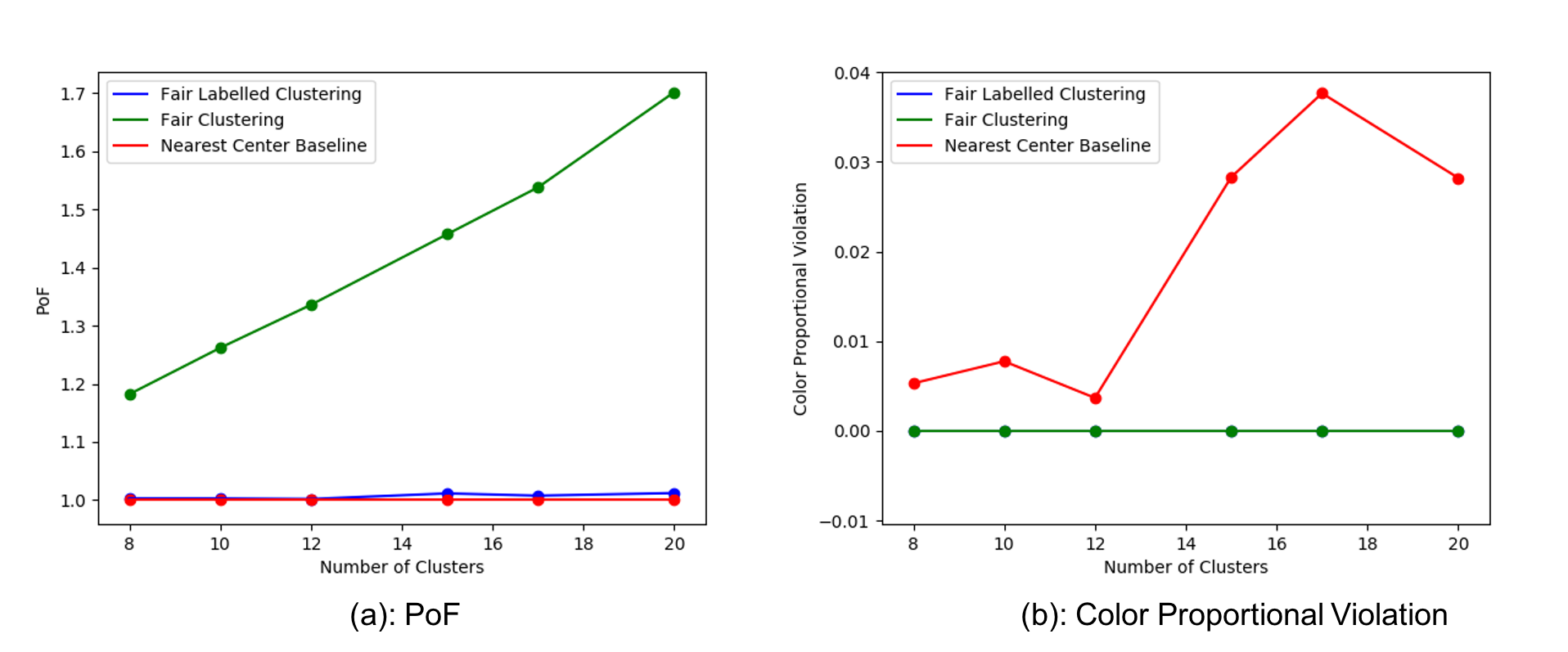}
\caption{\adult{} dataset results (a):$\POF$,  (b):$\Delta_{\text{color}}$}\label{adult_all_lcal}
\end{figure}

% %\begin{wrapfigure}[10]{r}{0.6\textwidth}
% \begin{figure}[h!]
% \centering
% \includegraphics[width=0.8\linewidth]{Figs/adult_POF.png}
% \caption{Number of clusters $k$ vs $\POF$ for $\adult$.}\label{adult_fig_1}
% \end{figure}
% %\end{wrapfigure}
\vspace{-0.4cm}
\paragraph{\credit{} \textbf{Dataset}:} Similar to the \adult{} dataset experiment, after finding the centers using $k$-means++, we assign them positive and negative labels. For similar motivations, if the center has a coordinate corresponding to the amount of balance that is $\ge 300{,}000$ we assign the center a positive label and a negative one otherwise. Figure \ref{credit_all_lcal} shows the results of the experiments. We see again that our algorithm leads to a lower price of fairness than fair clustering, but not to the same extent as in the \adult{} dataset but it still has no proportional violation just like fair clustering.

\begin{figure}[h!]
\centering
\includegraphics[width=0.6\linewidth]{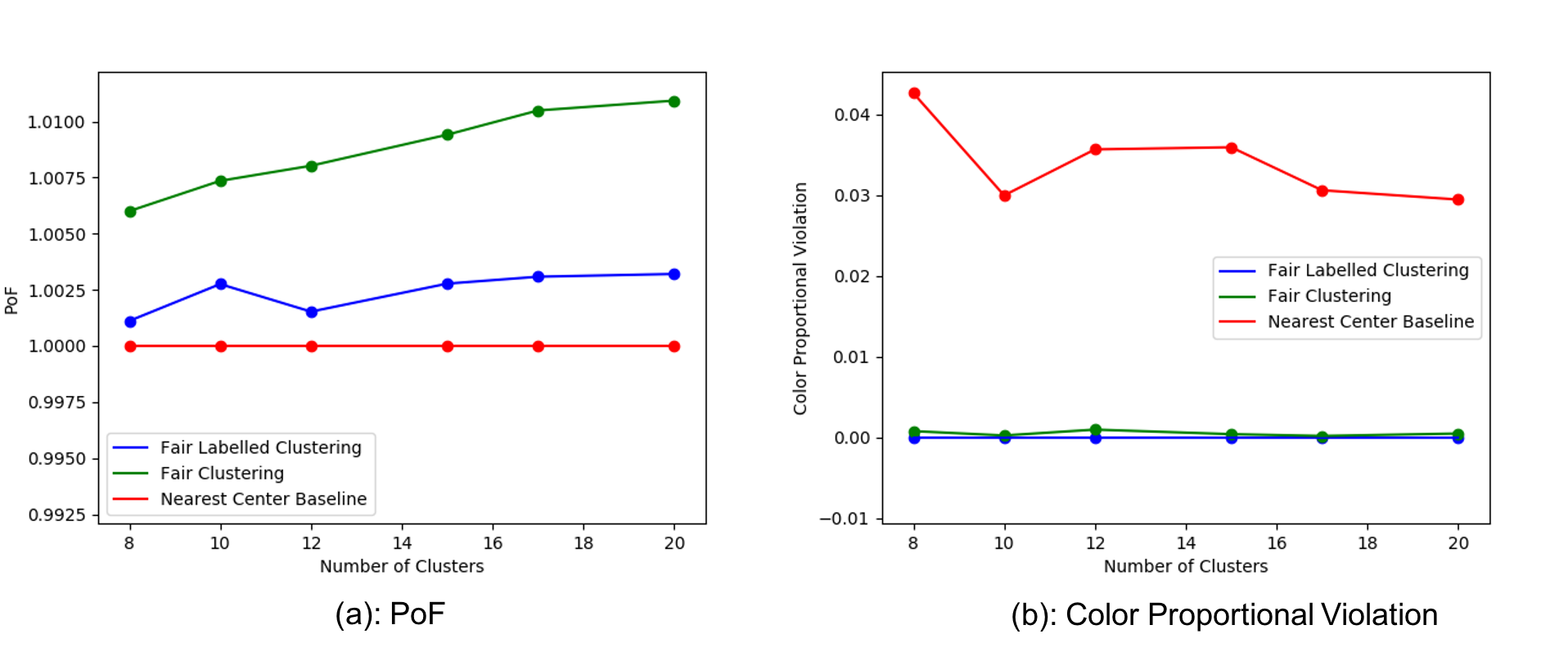}
\caption{\credit{} dataset results (a):$\POF$,  (b):$\Delta_{\text{color}}$}\label{credit_all_lcal}
\end{figure}

As mentioned in section \ref{sec:lcal_fast_two_label_alg}, algorithm (\ref{alg:alg_exact_k_med_means}) can allow the user to obtain the solutions for different values of $|\phi^{-1}(P)|$ (the number of points assigned to the positive set) without an asymptotic increase in the running time. In figure \ref{plot_vs} we show a plot of $|\phi^{-1}(P)|$ vs the clustering cost. Interestingly, requiring more points to be assigned to the positive label comes at the expense of a larger cost for some instances (\adult{} with $k=15$) whereas for others it has a non-monotonic behaviour (\adult{} with $k=10$). This can perhaps be explained by the different choices of centers as $k$ varies. There are $5$ centers with positive labels for $k=10$ ($50\%$ of the total), but only $4$ for $k=15$ (less than $30\%$) making it difficult to route points to positive centers.  

\begin{figure}[h!]
\centering
\includegraphics[width=0.6\linewidth]{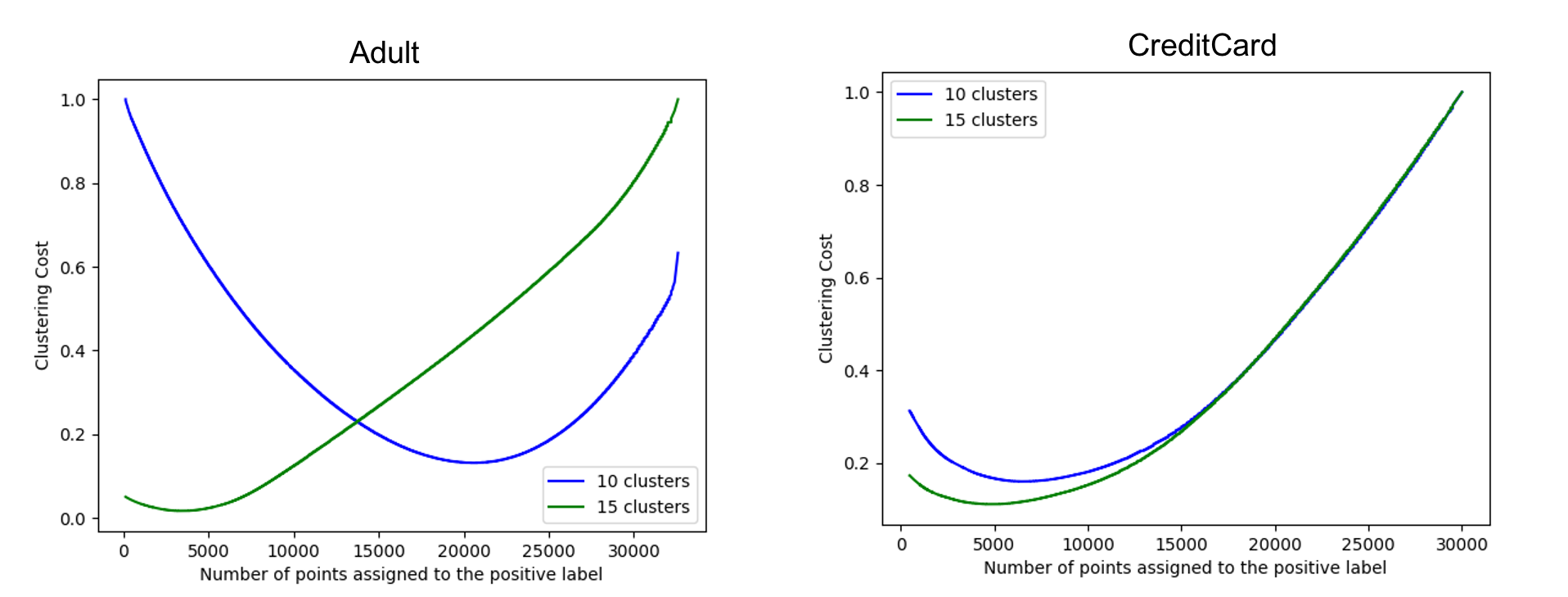}
\caption{A plot of $|\phi^{-1}(P)|$ vs the clustering cost (normalized by the maximum cost obtained).}\label{plot_vs}
\end{figure}

% \begin{figure}[h!]
% \centering
% \includegraphics[width=0.8\linewidth]{Figs/credit_POF.png}
% \caption{Number of clusters $k$ vs $\POF$ for $\credit$.}\label{credit_fig_1}
% \end{figure}

% \vspace{-.8cm}
% \paragraph{\textbf{A plot of $|\phi^{-1}(P)|$ vs the clustering cost:}} As mentioned in Section \ref{fast_two_label_alg}, a modification of algorithm (\ref{alg:alg_exact_k_med_means}) can allow the user to obtain the solutions for different values of $|\phi^{-1}(P)|$ (the number of points assigned to the positive set) without an asymptotic increase in the running time. In figure \ref{plot_vs} we show a plot of $|\phi^{-1}(P)|$ vs the clustering cost. Interestingly, requiring more points to be assigned to the positive label comes at the expense of a larger cost for some instances (\adult{} with $k=15$) whereas for others it has a non-monotonic behaviour (\adult{} with $k=10$). This can perhaps be explained by the different choices of centers as $k$ varies. There are $5$ centers with positive labels for $k=10$ ($50\%$ of the total), but only $4$ for $k=15$ (less than $30\%$) making it difficult to route points to positive centers.  

% \begin{figure}[h!]
% \centering
% \includegraphics[width=0.9\linewidth]{Figs/plot_vs_new.png}
% \caption{A plot of $|\phi^{-1}(P)|$ vs the clustering cost (normalized by the maximum cost obtained).}\label{plot_vs}
% \end{figure}

% We also note that in appendix \ref{appendix_scale} we investigate the scalability of our algorithm and we indeed observe that the algorithm scales well for large instances. 
\vspace{-0.1cm}
\subsection{$\LCUL$ Experiments}\label{exp_lcul}
Similar to the $\LCAL$ setting for $\LCUL$ we get the centers by running $k$-means++. However, we do not have the labels. We compare our algorithm (algorithm \ref{alg:alg_rand_lcul}) to two baselines: (1) Nearest Center with Random Assignment (\NCRA{}) and (2) Fair Clustering (\FC{}). We refer to our algorithm (block \ref{alg:alg_rand_lcul}) as \LFC{} (labeled fair clustering). In \NCRA{} we assign each point to its closest center which leads to an optimal clustering cost, whereas for fair clustering (\FC{}) we solve the fair clustering problem. For both \NCRA{} and \FC{} we assign each center label $L$ with probability $\alpha_L$.

We use two labels with $\alpha_1=\frac{1}{4}$ and $\alpha_2=\frac{3}{4}$. For all colors and labels we set $\epsilon^A_{h,L}={\epsilon'}^A_{h,L}=0.2$ and for all labels we set $\epsilon^B_{L}={\epsilon'}^B_{L}={\epsilon}^C_{L}={\epsilon'}^C_{L}=0.1$. Further, all algorithms satisfy the constraints in expectation, therefore we seek a measure of centrality around the expectation like the variance. Each algorithm is ran 50 times and we report the average values of $\Delta_{\text{color}}$,$\Delta_{\text{points/label}}$, and $\Delta_{\text{center/label}}$. 

\vspace{-0.3cm}
\begin{figure}[h!]
\centering
\includegraphics[width=0.5\linewidth]{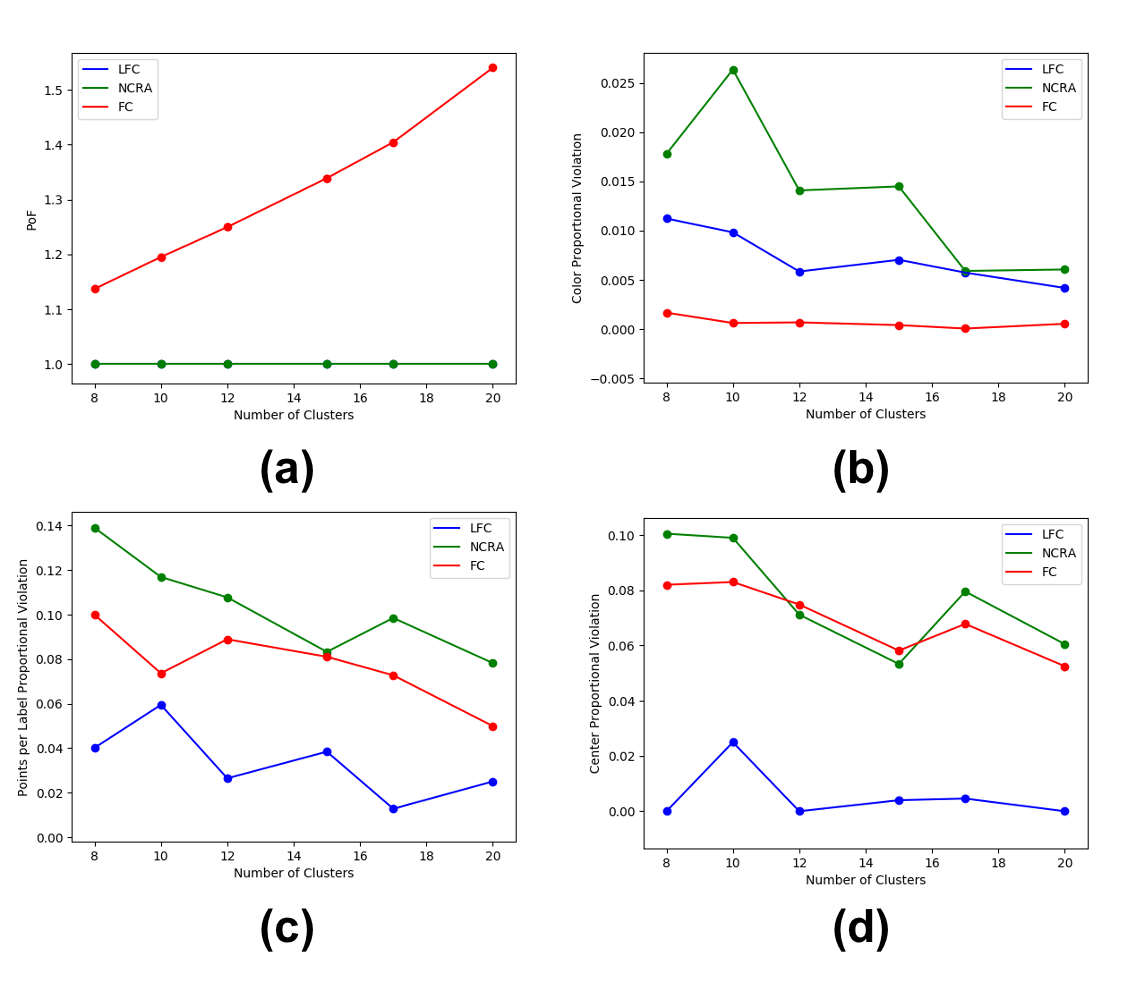}
\caption{$\LCUL$ results on the \adult{} dataset. (a):$\POF$,  (b):$\Delta_{\text{color}}$, (c):$\Delta_{\text{points/label}}$ ,(d):$\Delta_{\text{center/label}}$.}\label{adult_clul_fig}
\end{figure}

\vspace{-0.7cm}
\begin{figure}[h!]
\centering
\includegraphics[width=0.5\linewidth]{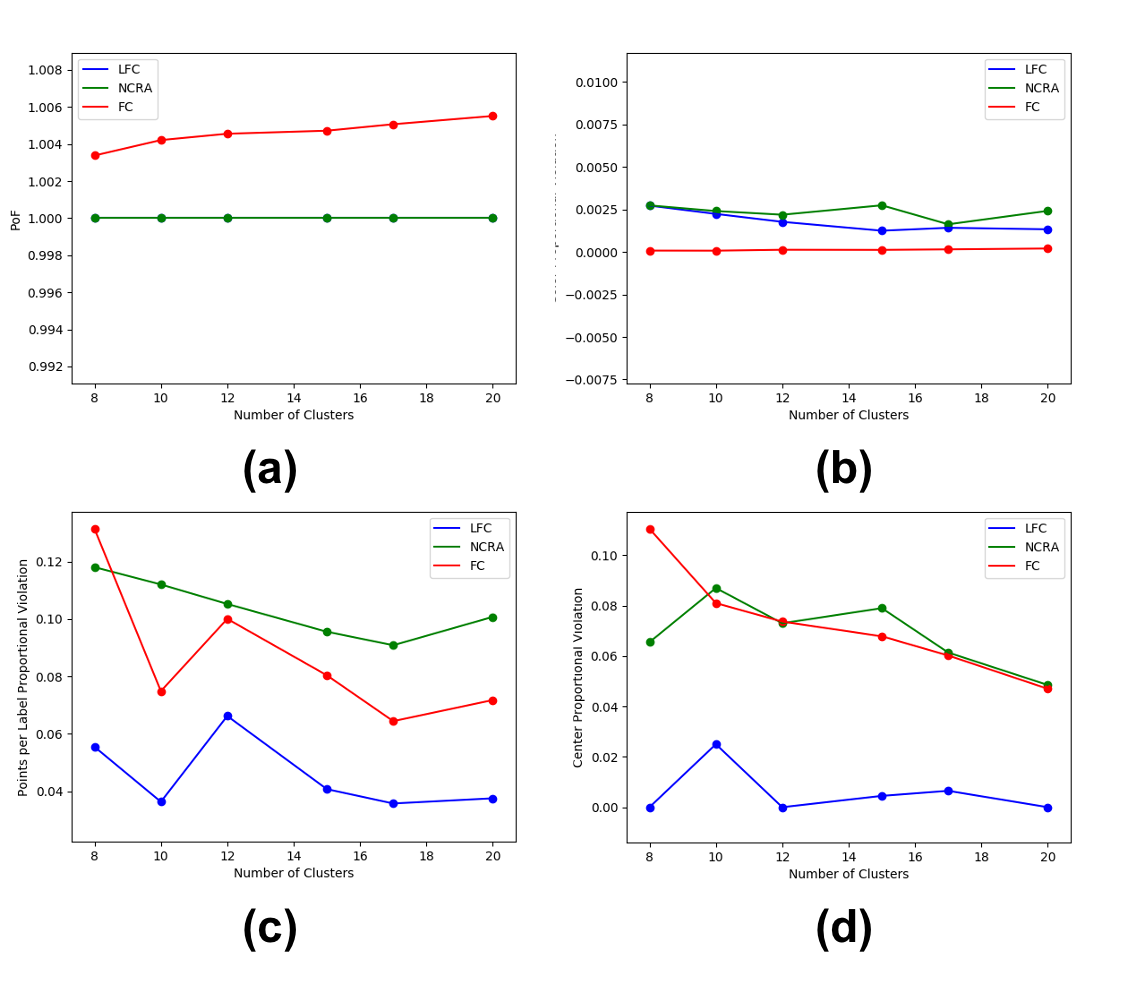}
\caption{$\LCUL$ results on the \credit{} dataset. (a):$\POF$,  (b):$\Delta_{\text{color}}$, (c):$\Delta_{\text{points/label}}$ ,(d):$\Delta_{\text{center/label}}$.}\label{credit_clul_fig}
\end{figure}
%\vspace{-0.2cm}
Figures \ref{adult_clul_fig} and \ref{credit_clul_fig} show the results for \adult{} and \credit{}. For $\POF$, our algorithm achieves an optimal clustering and hence coincides with \NCRA{} whereas fair clustering achieves a much higher $\POF$ as large as $1.5$. For the color proportionality ($\Delta_{\text{color}}$), we see that fair clustering has almost no violation whereas the \NCRA{} and labeled clustering have small but noticeable violations. For the number of points a label receives ($\Delta_{\text{points/label}}$) we notice that all algorithms have a violation although labeled clustering has a smaller violation mostly. As noted earlier, we suspect that this is a result of dependent rounding's negative correlation property leading to some concentration around the expectation. Finally, for the number of centers a label receives ($\Delta_{\text{center/label}}$), clearly \LFC{} has a much lower violation.

\vspace{-0.25cm}
\subsection{Algorithm Scalability}\label{appendix_scale}
Here we investigate the scalability of our algorithms. In particular, we take the \cens{} dataset which consists of 2,458,285 points and sub-sample it to a specific number, each time we find the centers with the $k$-means algorithm\footnote{We choose $k=5$ for all different dataset sizes.}, assign them random labels, and solve the $\LCAL$ and $\LCUL$ problems. Note since we care only about the run-time a random assignment of labels should suffice. Our group membership attribute is gender which has two values (two colors). We find our algorithm are indeed highly scalable (figure \ref{runtime_fig}) and that even for 500,000 points it takes less than 90 seconds. We note in contrast that the fair clustering algorithm of \cite{bera2019fair} would takes around 30 minutes to solve a similar size on the same dataset. In fact, scalability is an issue in fair clustering and it has instigated a collection of work such as \cite{huang2019coresets,backurs2019scalable}. The fact that our algorithm performs relatively well run-time wise is worthy of noting. 

\vspace{-0.2cm}
\begin{figure}[H]
\centering
\includegraphics[width=0.7\linewidth]{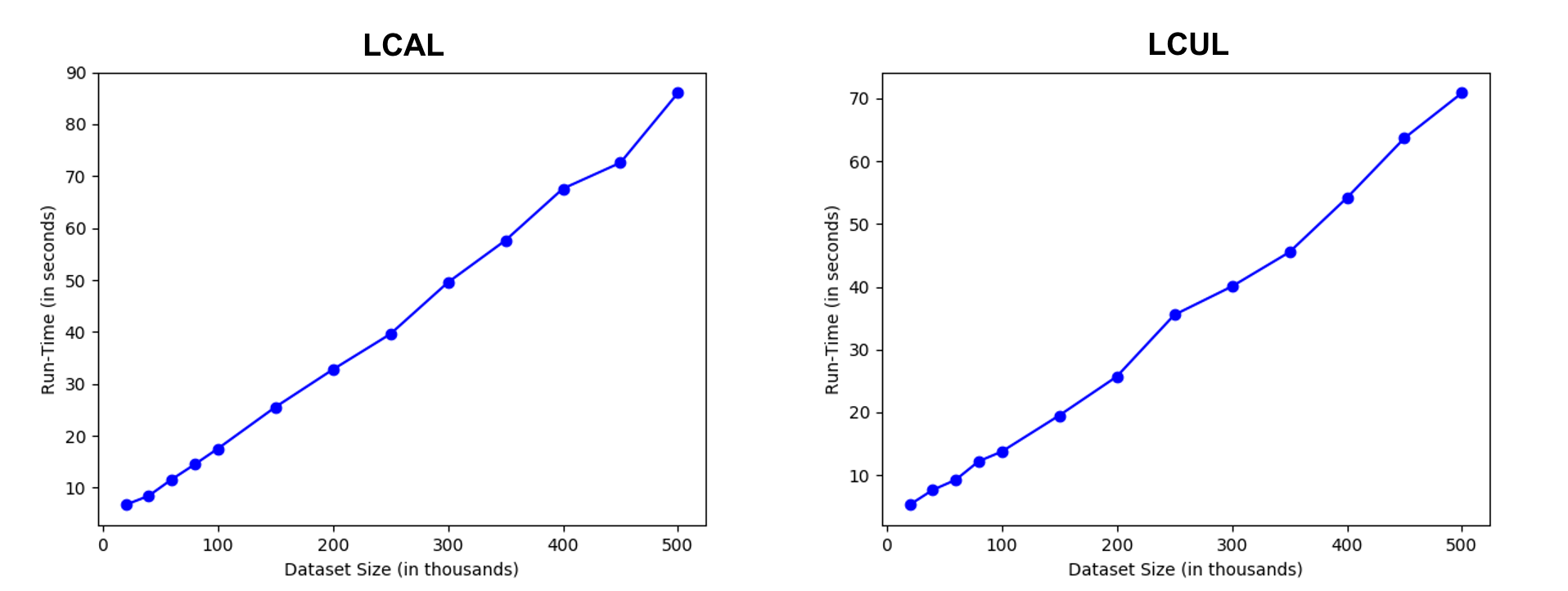}
\caption{Dataset size vs algorithm Run-Time: (left) $\LCAL$, (right) $\LCUL$.}\label{runtime_fig}
\end{figure}

% %\begin{wrapfigure}[10]{r}{0.6\textwidth}
% \begin{figure}[H]
% \centering
% \includegraphics[width=0.7\linewidth]{Figs/runTime_fig.png}
% \caption{$\LCAL$:Dataset size vs algorithm Run-Time.}\label{runtime_fig_lcal}
% \end{figure}

% \begin{figure}[H]
% \centering
% \includegraphics[width=0.7\linewidth]{Figs/LCUL_scaling.png}
% \caption{$\LCUL$:Dataset size vs algorithm Run-Time.}\label{runtime_fig_lcul}
% \end{figure}

%\end{wrapfigure}

% \begin{table*}[t]
% \caption{ Census-sex ($\delta$=0.2) , $k$ =5 }
% \label{table:census-sex}
% {\renewcommand{\arraystretch}{1.25}
% \begin{tabular}{l||c|c|c|c|c|c|c|c|c|c}
% \hline
%  \multicolumn{1}{l}{\textbf{No of Points}}  & \multicolumn{1}{l}{\textbf{20k}} & 
% \multicolumn{1}{l}{\textbf{40k}} & 
% \multicolumn{1}{l}{\textbf{60k}} &
% \multicolumn{1}{l}{\textbf{80k}} & 
% \multicolumn{1}{l}{\textbf{100k}} & 
% \multicolumn{1}{l}{\textbf{200k}} & 
% \multicolumn{1}{l}{\textbf{300k}} & 
% \multicolumn{1}{l}{\textbf{400k}} & 
% \multicolumn{1}{l}{\textbf{500k}} 
% \\ \hline
% \textbf{Time} &
% 6.692967 & 8.415920 &11.580700 & 
% 14.471128 & 
% 17.487776 &
% 32.736872 &
% 49.514137 &
% 67.585873 &
% 86.032944
%  \\ \hline

% \end{tabular}
% }
% \end{table*}
\section{Conclusion}
Motivated by fairness considerations and the quality of outcome each cluster receives, we have introduced fair labeled clustering. We showed algorithms for the case where the centers' labels are decided and have shown that unlike fair clustering we end up with a much lower cost while still satisfying the fairness constraints. For the case where the centers' labels are not decided we gave a detailed characterization of the complexity and showed an algorithm for a special case. Experiments have shown that our algorithms are scalable and much faster than fair clustering. 
% Future work can generalize the setting where the centers are undecided.
\section{Acknowledgments}   % Seyed -- if this counts for page limit, let me know, I can shorten it substantially.
This research was supported in part by 
NSF CAREER Award IIS-1846237,  % John
NSF Award CCF-1749864,  % Aravind
NSF Award CCF-1852352,  % John 
NSF Award SMA-2039862,  % John
NIST MSE Award \#20126334,  % John 
DARPA GARD \#HR00112020007,  % John 
DoD WHS Award \#HQ003420F0035,  % John 
DARPA SI3-CMD \#S4761,  % John 
ARPA-E DIFFERENTIATE Award \#1257037,  % John 
and gifts by research awards from
Adobe, % Aravind
Amazon, % Aravind
and Google.  % Aravind+John
% 
% \begin{table*}[t]
% \caption{ Census-sex ($\delta$=0.2) , $k$ =5 }
% \label{table:census-sex}
% {\renewcommand{\arraystretch}{1.25}
% \begin{tabular}{l||c|c|c|c|c|c|c|c|c|c}
% \hline
%  \multicolumn{1}{l}{\textbf{No of Points}}  & \multicolumn{1}{l}{\textbf{20k}} & 
% \multicolumn{1}{l}{\textbf{40k}} & 
% \multicolumn{1}{l}{\textbf{60k}} &
% \multicolumn{1}{l}{\textbf{80k}} & 
% \multicolumn{1}{l}{\textbf{100k}} & 
% \multicolumn{1}{l}{\textbf{200k}} & 
% \multicolumn{1}{l}{\textbf{300k}} & 
% \multicolumn{1}{l}{\textbf{400k}} & 
% \multicolumn{1}{l}{\textbf{500k}} 
% \\ \hline
% \textbf{Time} &
% 6.692967 & 8.415920 &11.580700 & 
% 14.471128 & 
% 17.487776 &
% 32.736872 &
% 49.514137 &
% 67.585873 &
% 86.032944
%  \\ \hline

% \end{tabular}
% }
% \end{table*}

\bibliographystyle{ACM-Reference-Format}
\bibliography{refs}

%%%%%%%%%%%%%%%%%%%%%%%%%%%%%%%%%%%%%%%%%%%%%%%%%%%%%%%%%%%%%%%%%%%%%%%%
\clearpage
\newpage 
\appendix
\section{Omitted Proofs}\label{app_missing_proofs}
We note that all of our hardness results use the $k$-center problem for simplicity. Before we introduce the hardness result, we note all of our reductions are from exact cover by 3-sets (X3C) \cite{garey1979computers} where we have universe $\mathcal{U}=\{u_1,u_2,\dots,u_{3q}\}$ and subsets $\mathcal{W}_1,\dots,\mathcal{W}_t$ where $t=q+r$ and for non-trivial instances $r>0$. We form an instance of $\LCUL$ by representing each one the subsets $\mathcal{W}_1,\dots,\mathcal{W}_t$ by a vertex and each element in $\mathcal{U}=\{u_1,u_2,\dots,u_{3q}\}$ by a vertex. The centers are the sets $\mathcal{W}_1,\dots,\mathcal{W}_t$ and they are given a blue color whereas the rest of the points (in $\mathcal{U}$) are red. Further, each point $u_i$ is connected by a edge to a center $\mathcal{W}_i$ if and only if $u_i \in \mathcal{W}_j$. The distances between any two points is the length of the shortest path between them. This clearly leads to a metric. See figure \ref{hardness_fig_1} for an example. This is essentially a reduction we follow in all proofs, sometimes changes are introduced and mentioned explicitly in the proofs. 

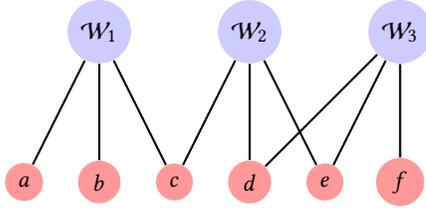
\begin{figure}[h!] 
\begin{tikzpicture}
\tikzstyle{level1}= [circle,fill=blue!20]
\tikzstyle{level2}= [circle,fill=red!40]
\tikzstyle{edge} = [-, thick]
\node[level1] (w1) at (0,4) {$\mathcal{W}_1$};
\node[ level1] (w2) at (2,4) {$\mathcal{W}_2$};
\node[level1 ] (w3) at (4,4) {$\mathcal{W}_3$};
\node[level2 ] (a) at (-1,2) {$a$};
\node[level2] (b) at (0,2) {$b$};
\node[level2 ] (c) at (1,2) {$c$};
\node[level2 ] (d) at (2,2) {$d$};
\node[level2 ] (e) at (3,2) {$e$};
\node[level2] (f) at (4,2) {$f$};

\draw[edge] (w1)--(a);
\draw[edge] (w1)--(b);
\draw[edge] (w1)--(c);
\draw[edge] (w2)--(c);
\draw[edge] (w2)--(d);
\draw[edge] (w2)--(e);
\draw[edge] (w3)--(d);
\draw[edge] (w3)--(e);
\draw[edge] (w3)--(f);

\end{tikzpicture}
\caption{Example of the reduction for theorem (\ref{hardness_1}). This is an instance of the LUCL problem for an instance $U= \{a,b,c,d,e,f\}$, $\mathcal{W}_1 = \{a,b,c\}$, $\mathcal{W}_2= \{c,d,e\}$ and $\mathcal{W}_3=\{d,e,f\}$  with $q =2$, $|U|=3q$ and $t=3$. }
\label{hardness_fig_1}
\end{figure}

Now we introduce the following theorem:
\begin{theorem}\label{hardness_a}
Even if the color-proportionality constraint (\ref{lcul_1}) are ignored\footnote{We can simply remove the constraint or set $l^{L}_{\pcolor}=0,u^{L}_{\pcolor}=1, \forall \pcolor \in \Colors, L \in \labs$.} $\LCUL$ is NP-hard. 
\end{theorem}
%%%%%%%%%%%%% HERE %%%%%%%%%%%%%%%%%%%%%%%%%%%%%%%%%

\begin{proof}
As mentioned we consider an instance of exact cover by 3-sets (X3C) with universe $\mathcal{U}=\{u_1,u_2,\dots,u_{3q}\}$ and subsets $\mathcal{W}_1,\dots,\mathcal{W}_t$. We construct an instance of $\LCUL$ where the proportionality constraints are ignored. Further, we only have two labels $\labs=\{N,P\}$, we set $\cl{P}=0,\cu{P}=q$,$\cl{N}=0,\cu{N}=t$ and $\lb{P}=4q,\ub{P}=3q+t$, $\lb{N}=0,\ub{N}=3q+t$. 
% We will have $t$ many centers one for each subset $\mathcal{W}_1,\dots,\mathcal{W}_t$. The set of centers is called $S$. Further, we will have $3q$ many vertices (points), one for each element in $\mathcal{U}$, our final set of points $\Points$ will consist of these $3q$ points along with the $t$ many points for the centers. If an element $u_{\point} \in \mathcal{U}$ is contained in set $\mathcal{W}_i$, then the distance between their corresponding vertices in the $\LCUL$ instance is set to 1, i.e. $d(\point,i)=1, \text{ if } u_{\point} \in \mathcal{W}_i$, otherwise the distance is set to 2. The distance between the points $\point,\point' \in \Points$ is $d(\point,\point')=2 \text{ if } \point \neq \point' \text{ and } d(\point,\point')=0 \text{ otherwise}$. Similarly, for the distances between the centers we have: $\forall i,i' \in S$ is $d(i,i')=2 \text{ if } i \neq i' \text{ and } d(i,i')=0 \text{ otherwise}$. It can be verified that such an distance assignment leads to a metric. Since the smallest distance between any center and a point in $\mathcal{U}$ is $\ge 1$, the optimal value for $\LCUL$ for the $k$-center objective cannot be less than $1$. 

A solution for X3C leads to a solution for $\LCUL$ at cost $1$: Take the collection of $q$ many subsets that solve X3C and give their corresponding centers in $\LCAL$ a positive label. Then it is clear that $|S^P|=q$ and that the number of points covered by the positive centers is $4q$ and that this done at a cost of 1. The centers that do not correspond to the solution of X3C will be given a negative label and assigned no points. 

A solution for $\LCUL$ at cost $1$ leads to a solution X3C: A solution for $\LCUL$ cannot assign more than $\cu{P}=q$ many centers a positive label and it has to cover $3q$ more points to have a total of $4q$ points and this has to be done at a distance of 1. By construction, since each center is connected to 3 points, the $\LCUL$ solution cannot have less than $q$ centers. Further, to have $4q$ points, then each center would have to cover a unique set of 3 points at a distance of 1. Since points are connected to centers at a distance of 1 only if they are corresponding values are contained in the subsets corresponding to those centers, it follows that the $q$ subsets in the $\LCUL$ solution are indeed an exact cover for X3C.  

% See figure \ref{hardness_fig_1} for an example of this reduction. 
\end{proof}

Here we instead we ignore the constraints on the number of points a label should receive, i.e. constraints (\ref{lcul_2} and keep the proportionality constraints. We show that this also results in an NP-hard problem as demonstrated in the theorem below:
\begin{theorem}\label{hardness_b}
Even if we do not specify the number of points a label should receive (constraint(\ref{lcul_2})), $\LCUL$ is NP-hard.   
\end{theorem}
\begin{proof}
Similar to the proof of theorem (\ref{hardness_a}) we follow the reduction from X3C with two labels for $\LCUL$, i.e. $\labs=\{N,P\}$, but now we consider the color of the vertices. Vertices of the subsets $\mathcal{W}_1,\dots,\mathcal{W}_t$ are blue and all of the vertices of the elements of $\mathcal{U}$ are red. For the $\LCUL$ instance, we set $\cl{P}=q,\cu{P}=t$,$\cl{N}=0,\cu{N}=t$. The representation for the negative set is ignored, i.e. $l^{N}_{\text{red}}=l^{N}_{\text{blue}}=0$ and $u^{N}_{\text{red}}=u^{N}_{\text{blue}}=1$. For the positive set, we only have set a bound on the lower proportion for the red color, specifically $l^{P}_{\text{red}}=\frac{3}{4}, u^{P}_{\text{red}}=1$ and $l^{P}_{\text{blue}}=0,u^{P}_{\text{blue}}=1$. As the reduction of theorem (\ref{hardness_a}) the optimal value of the $k$-center objective cannot be less than 1. 

A solution for X3C leads to a solution for $\LCUL$ at cost $1$: Take the $q$ subsets in the solution of X3C and assign their corresponding centers a positive labels, then $|S^P|=q\ge \cu{P}$. Further since elements of $\mathcal{U}$ are represented by red vertices, you will have $3q$ red vertices covered at a distance of 1, the red proportion of the positive label would be $\frac{3q}{4q}=\frac{3}{4}\ge l^{P}_{\text{red}}$. To complete the solution assign the rest of the centers a negative label.

A solution for $\LCUL$ at cost $1$ leads to a solution X3C: A solution for $\LCUL$ would have to choose at least $\cl{P}=q$ many centers. Since all centers are blue and because there are only $3q$ many red points in the graph, we would have to choose exactly $q$ centers and cover all of the $3q$ many red points to satisfy the color proportionality constraints of $l^{P}_{\text{red}}$. Since this is being done at a cost of 1, these points must be representing elements in $\mathcal{U}$ that are contained in the subsets corresponding to the selected centers. Further, since every center is connected to exactly 3 points at radius 1, we have found an exact cover.  
%See figure \ref{hardness_fig_1} for a representation of this reduction. 
\end{proof}

\begin{theorem}\label{hardness_c}
Even if we do not specify the number of centers of each label (ignoring constraints (\ref{lcul_3}) ), $\LCUL$ is NP-hard.  
\end{theorem}
\begin{proof}
Similar to theorems (\ref{hardness_a},\ref{hardness_b}) we follow the same reduction from X3C. This time we ignore constraint (\ref{lcul_3}) on the number of centers, i.e. $0 \leq |S^N|,|S^P| \leq k$. We set $\lb{P}=\ub{P}=4q$ and $\lb{N}=0, \lb{N}=n$. Further for the color proportionality constraints, we have for the positive set we set $l^{P}_{\text{red}}=u^{P}_{\text{red}}=\frac{3}{4}$,  $l^{P}_{\text{blue}}=u^{P}_{\text{blue}}=\frac{1}{4}$ and for the negative set we have $l^{N}_{\text{red}}=l^{N}_{\text{blue}}=0$. $u^{N}_{\text{red}}=u^{N}_{\text{blue}}=1$.

A solution for X3C leads to a solution for $\LCUL$ at cost $1$: Simply let the subsets (centers) in the solution if X3C have a positive label and assign all of the points in $\mathcal{U}$ to them. Clearly, we have $\lb{P}=\ub{P}=4q$ and the red color has a representation of $\frac{3}{4}$ and the blue has a representation of $\frac{1}{4}$. Furthe, this is done at an optimal cost of 1.

A solution for $\LCUL$ at cost $1$ leads to a solution X3C: Since $\lb{P}=\ub{P}=4q$, $l^{P}_{\text{red}}=u^{P}_{\text{red}}=\frac{3}{4}$, and $l^{P}_{\text{blue}}=u^{P}_{\text{blue}}=\frac{1}{4}$, it follows that the positive set should cover $\frac{3}{4}4q=3q$ many red points and that it must also cover $\frac{1}{4}4q=q$ many blue points. Since all blue points are centers and all red points are from $\mathcal{U}, $it follows that we have to choose $q$ many centers to cover $3q$ many points at an optimal cost of 1. This leads to a solution for X3C. 
\end{proof}

Now we re-state the original theorem from the main paper:
\thmhardnessone*

\begin{proof}
This follows immediately from theorems (\ref{hardness_a},\ref{hardness_b},\ref{hardness_c}) above.
\end{proof}

%We now re-state theorem(\ref{hardness_4}) followed by its proof:
\hardnessfour*
\begin{proof}
This follows simply by noting that if the labels are assigned, then we have an $\LCAL$ instance which solvable in time that is polynomial in $n$ and $k$, since $k\leq n$, it follows that the run time for solving $\LCAL$ is $O(n^c)$ for some constant $c$. Now, since there are at most $m^k$ many label choices for the centers, it follows that the run time is for  $\LCUL$ is $O(m^k n^c)$.  
\end{proof}

%We now re-state theorem(\ref{hardness_5}) followed by its proof:
\hardnessd*
\begin{proof}
Let us consider the color proportionality constraint (\ref{lcul_1}) alone. To solve the problem optimally and satisfy the constraint, simply assign all points to their closest center and let all centers take one label from the set $\labs$. 

Now, we consider only the constraints on the number of centers for each label (\ref{lcul_3}). Again we assign each point to its closest center for an optimal cost. To satisfy constraints (\ref{lcul_3}), assuming the constraint parameters of (\ref{lcul_3}) lead to a feasible problem, then each label $L \in \labs$, assign it $\cl{L}$ many centers arbitrarily. If some centers have not been assigned any labels, then simply go to label $L$ which has not reached its upper bound $\cu{L}$ and assign more labels from it. We simply keep assigning labels from label values that have not reached their upper bound on the number of centers until all centers have a label. 

Now, we consider only the constraints on the number of points a label receives (\ref{lcul_2}). We simply follow the same reduction from theorems (\ref{hardness_a},\ref{hardness_b},\ref{hardness_c}), see also the beginning of this subsection for the details of the reduction from X3C. We have $t=q+r$ many subsets, we let the number of labels of the $\LCUL$ instance be $m=t=q+r$. Further, we partition the set of labels into two, i.e. $\labs=\labs_1 \cup \labs_2$ where $|\labs_1|=q$ and $|\labs_2|=r$, and we set the lower and upper bounds for the labels according to these sets. Specifically, $\forall L \in \labs_1: \lb{L}=\ub{L}=4q$ and $\forall L \in \labs_2: \lb{L}=\ub{L}=1$. Now, clearly a solution for X3C leads to a solution for the $\LCUL$ instance, we simply let the subsets (centers) in the solution of X3C be the centers for the label set $\labs_1$. Each center is assigned a label from $\labs_1$ and covers itself and 3 points from $\mathcal{U}$, this leads to $4q$ many points which clearly satisfies the upper and lower bounds. Further, the centers not the solution are assigned a label from $\labs_2$ and cover themselves, which is just 1 point and therefore satisfies the constraints. Now for the reverse direction, consider the set $\labs_2$ where we have $r$ many labels each covering 1 point. It clear, the smallest cost would be for a center to be assigned to itself, it follows that we are looking for $r$ many centers and that each center should only be assigned to itself. This then leaves us with $q$ many centers, since no center can cover more than $4q$ many points at a distance of 1, and since we have $q$ many labels with each having to cover $4q$ many points, we clearly have a set cover, i.e. a solution for X3C. 
\end{proof}

%\subsection{Proofs for section \ref{sec:lcul_algs}}
%We now re-state theorem(\ref{twocolorhardness}) followed by its proof:
\twocolorhardness*

\begin{proof}
We follow a reduction for X3C (see the beginning of the appendix). We consider the two label case, $\labs=\{N,P\}$. Similiar to the previous reductions we will have $t$ many blue centers for the subsets $\mathcal{W}_1,\dots,\mathcal{W}_t$ each being connected to its elements in $\mathcal{U}$ at a distance of 1 with all elements in $\mathcal{U}$ being red. Note that $|\mathcal{U}|=q$ and that $t=q+r$. Now we also add $2q$ many blue centers which are not connected to anything by an edge, expect for one center which is connected by an edge to a new $3(r+2q)$ many red points, this means that any one of these red points is at a distance of 1 from this new center. Note that the increase in the problem size is still polynomial in the original X3C problem. We set the color proportionality constraint so that each label should have exactly 3:1 ratio of red points to blue points. Now the total number of points in the problem is $n=4q+r+2q+3(r+2q)=4(3q+r)$. The number of centers $k=q+r+2q=3q+r$. Further, we set $\alpha_P=\frac{q}{(3q+r)}$ and $\alpha_N=1-\alpha_P=\frac{2q+r}{3q+r}$. We set the lower and upper size bounds according to $\alpha_P$ and $\alpha_N$, this leads to $\lb{P}=\ub{P}=\alpha_P n=\frac{q}{(3q+r)} n=\frac{q}{(3q+r)}4(3q+r)=4q$ and $\lb{N}=\ub{N}=\frac{2q+r}{3q+r}n=\frac{2q+r}{3q+r}4(3q+r)=4(2q+r)$. Further, the number of centers for each label are $\cl{P}=\cu{P}=\alpha_P k=\frac{q}{(3q+r)}k=\frac{q}{(3q+r)}3q+r=q$ and $\cl{N}=\cu{N}=\alpha_N k=\frac{2q+r}{3q+r}3q+r=2q+r$. 

A solution for X3C leads to a solution for $\LCUL$ at cost $1$: Simply let the $q$ many centers representing the solution set in $\mathcal{W}_1,\dots,\mathcal{W}_t$ be the positive labeled centers and assign them the points that belong to them and let all other centers be negative and assign the last new center all of the $3(r+2q)$ many red children points. We then $q$ many positive centers covering $4q$ many points with the color proportionality being 3:1 red points to blue points. Similarly, for the negative set we have $2q+r$ many centers covering $4(2q+r)$ many points at a color proportionality of 3:1 red to blue. This is done at cost of 1, so clearly optimal.

A solution for $\LCUL$ at cost $1$ leads to a solution X3C: Suppe the new blue center with $3(r+2q)$ many red children is assigned a positive label, this to achieve an optimal cost all of its children have to be assigned to it. This means that the positive set would have at least $3(r+2q)=6q+3r$ many points, but $\lb{P}=\ub{P}=\alpha_P n = 4q < 6q < 6q+3r$ which causes a contradiction. Therefore that center can never be positive. Therefore, we are looking for $\alpha_P k =q$ many centers to cover $\alpha_P n= 4q$ many points and because of the color proportionality constraint $3q$ many of them are red and $q$ are blue. Finding this set at an optimal cost is a solution for X3C.    
\end{proof}

\section{$\LCAL$ Algorithm for Two Labels and General Proportions}\label{gen_alg}
Our algorithm for general proportions is similar to the exact preservation algorithm. The two main differences lie in the fact that we use feasibility checks for a given value of $n_P$ where $n_P$ is the number of points to be assigned to the positive label and that the way we move points from the negative-labelled centers to the positive-labelled labels is more elaborate. In particular algorithm block \ref{alg:alg_non_exact_k_med_means} shows our algorithm. Note that $n_{\pcolor}=|\Points^{\pcolor}|$. 

\begin{algorithm}[h!]
  \caption{Non-Exact Preservation for $k$-median/$k$-means}
  \label{alg:alg_non_exact_k_med_means}
\begin{algorithmic}[1]
  \STATE Define the optimal assignment as $\phi^{*}$ and its cost as $cost^{*}$.
  \STATE \textbf{Step 1:} 
  \STATE Find an assignment $\phi_0$ that assigns all points to their nearest center in $N$, this means that $|\phi^{-1}_{0}(N)|=n$ and $|\phi^{-1}_{0}(P)|=0$. Update $\phi^{*}$ and $cost^{*}$ according to the values for this solution $\phi_0$. 
  \STATE \textbf{Step 2:} 
  \FOR{$n_P=0$ to $n$ }  
      \STATE $\forall \pcolor \in \Colors$ find $n^P_{\pcolor,l}= \max \Big(\ceil{l^P_{\pcolor} n_P},n_{\pcolor}-\floor{u^N_{\pcolor}(n-n_P)} \Big)$
      \STATE $\forall \pcolor \in \Colors$ find $n^P_{\pcolor,u}= \min \Big(\floor{u^P_{\pcolor} n_P},n_{\pcolor}-\ceil{l^N_{\pcolor}(n-n_P)} \Big)$
      \STATE Find $n_{P,l}=\sum_{\pcolor \in \Colors} n^P_{\pcolor,l}$ and $n_{P,u}=\sum_{\pcolor \in \Colors} n^P_{\pcolor,u}$
      \IF{ ($\forall \pcolor \in \Colors: n^P_{\pcolor,u} \ge n^P_{\pcolor,l}$) \AND ($ n_{P,u} \ge n_P \ge n_{P,l}$)}
             \STATE $\forall \pcolor \in \Colors$ move as many points with the maximum drop such that there at least $n^P_{\pcolor,l}$ points of color $\pcolor$ in the positive set 
             \IF{$n_P > n_{P,l}$} 
                 \STATE Move $\big(n_P-n_{P,l}\big)$ many points to the positive set, each time selecting the point with the maximum drop provided the total number in the positive set of its color $\pcolor$ does not exceed $n^P_{\pcolor,u}$. 
                 \STATE For each point $\point$ moved to the positive set, record $N(j)$, $P(j)$, and the iteration of movement $t_j=i$. 
             \ENDIF
             \STATE Record the cost in entry $cost(i)=cost_i$. 
                 \IF{$cost_{i} < cost^{*}$}
                    \STATE Let $\phi^{*}=\phi_{i}$ and $cost^{*}=cost_i$ \COMMENT {Update $\phi^{*}$ and $cost^{*}$}
                \ENDIF
      \ELSE 
            \STATE Move to line 5. 
      \ENDIF
  \ENDFOR  
  \RETURN $\phi^{*}, cost^{*}$
\end{algorithmic}
\end{algorithm}

Notice how we calculate $n^P_{\pcolor,l}$ and $n^P_{\pcolor,u}$ in each iteration, these are the lower and upper bounds for the number of points of color $\pcolor$ that should be in the positive set for a given value of $n_P$. It is not difficult to see that a violation of these bounds would cause a problem in the proportion representation in either the positive of negative set and that $n_{P,u} \ge n_P \ge n_{P,l}$. This is proved in the next lemma:
\begin{lemma}\label{feasible_move}
If the value of $n_P$ satisfies: (1) $\forall \pcolor \in \Colors: n^P_{\pcolor,u} \ge n^P_{\pcolor,l}$. (2) $n_{P,u} \ge n_P \ge n_{P,l}$. 
Then $n_P$ is feasible.
\end{lemma}
\begin{proof}
For (1): if we have $\forall \pcolor \in \Colors: n^P_{\pcolor,u} \ge n^P_{\pcolor,l}$, then we can have $n_{P,h}$ points such that $n^P_{\pcolor,u} \ge n_{P,h} \ge n^P_{\pcolor,l}$. It follows by the values of $n^P_{\pcolor,l}$ and $n^P_{\pcolor,u}$ that $\forall \pcolor \in \Colors: \floor{u^P_{\pcolor} n_P}  \ge n_{P,h} \ge \ceil{l^P_{\pcolor}n_P}$ which means that the positive set is indeed feasible in terms of proportions. 

For the negative set, we would have $n_N=n-n_P$ many points in the negative set as well as $n_{N,h} = n_{\pcolor} - n_{P,h}$ many points of color $\pcolor$ in the negative set. It follows, that $n_{N,h} = n_{\pcolor} - n_{P,h} \ge n_{\pcolor} - \min \Big(\floor{u^P_{\pcolor} n_P},n_{\pcolor}-\ceil{l^N_{\pcolor}(n-n_P)} \Big) \ge n_{\pcolor} - \big( n_h -\ceil{l^N_{\pcolor}(n-n_P)} \big) \ge \ceil{l^N_{\pcolor}(n-n_P)} \ge \ceil{l^N_{\pcolor}n_N}$. By similar arguments we can show that $n_{N,h} \leq \floor{u^N_{\pcolor}n_N}$ which means that the color proportions are balanced for the negative set as well. 

For (2): It is immediate to see that this is feasible if $n_{P,u}=n_P= n_{P,l}$, if on the other hand we have an inequality on one side, then we can also have a total of $n_P$ by moving points for each color within its upper and lower bounds ($n^P_{\pcolor,l}$ and $n_{P,u}$) until we have a total of $n_P$ many points.
\end{proof}

\begin{theorem}
Algorithm (\ref{alg:alg_non_exact_k_med_means}) returns an optimal solution. 
\end{theorem}
\begin{proof}
It is clear that when all points are assigned to the negative label the solution is optimal for that value of $n_P$ ( although possibly not feasible). As we iterate through the values of $n_P$ until we find a feasible $n_P$ our method results in feasible solution as proved in lemma (\ref{feasible_move}) and it also leads to an optimal solution for that value of $n_P$. To prove the second statement follow a similar argument to the proof of theorem (\ref{th_exact_alg_correct}).  

Specifically, suppose that we are iteration $n_P$ and that the last time the solution was updated \footnote{Note that the solution at step $(n_P-1)$ may not be updated since the number of points $n_P$ assigned to the positive set of centers may not be feasible.} was at iteration $n_P'<n_P$. Assuming our solution at step $n_P'$ is optimal we wish to prove that the solution for step $n_P$ is also optimal. Let $\phi'$ and $\phi$ denote the assignments for steps $n_P'$ and $n_P$, respectively. As the assignment changes from $\phi'$ to $\phi$, we can put the points in 4 sets: $\Points_{N\rightarrow N},\Points_{P\rightarrow P},\Points_{N\rightarrow P},\Points_{P\rightarrow N}$. The first two sets remain assigned to the same label, whereas the last two change labels. Since the first two set of points do not change labels, they are assigned to the same centers as they were in $\phi'$ since that is their closest center in $N$ or $P$. Since at iteration $n_P$, we should have $n_P$ many points assigned to the positive set and at iteration $n_P'$ we had $n_P'$ many points assigned, then $|\Points_{N\rightarrow P}|=|\Points_{P\rightarrow N}|+(n_P-n_P')$. Further, let $n'_h$ denote the number of points of color $h$ assigned to the positive label by the assignment $\phi'$. 

% Suppose that $|\Points_{N\rightarrow P}|=\nfair$, this implies that $|\Points_{P\rightarrow N}|=0$. Since points in $\Points_{N\rightarrow N}$ and $\Points_{P\rightarrow P}$ do not change their assignment, it follows that to get the minimum cost fair solution we have to move $\nfaircol{}$ many points for color $\pcolor$ from the negative set to the positive set. It is clear that this is achieved by moving $\nfaircol{}$ from each color $\pcolor$ with the maximum drop in cost. Therefore, the algorithm obtains an optimal solution. 

%Suppose instead that $|\Points_{P\rightarrow N}|=t>0$, then 

Sort the set of points in $\Points_{N\rightarrow P}$ descendingly according to their drop value, take from each color $n^P_{h,l}-n'_h$ many points with the maximum drop. Further, if this set does not have a size of $n_P$, then choose more points from each color provided their upper bound has not be reached, each time picking the ones with the maximum drop, let $\Points^{*}_{N\rightarrow P}$ denote that resulting set of points, and $\Points^{*\pcolor}_{N\rightarrow P}$ the subset of $\Points^{*}_{N\rightarrow P}$ with color $\pcolor$. It follows that the change in the cost is: 
\begin{align}
    & \underbrace{ \sum_{\point \in \Points^{*\pcolor_1}_{N\rightarrow P}} \Big( N(j)-P(j) \Big) + \dots + \sum_{\point \in \Points^{*\pcolor_{\colSize}}_{N\rightarrow P}} \Big( N(j)-P(j) \Big)  }_{A} + \\ 
    & \underbrace{ \sum_{\point \in \Points_{N\rightarrow P}-\Points^{*}_{N\rightarrow P}} \Big( N(j)-P(j) \Big) + \sum_{\point \in \Points_{P\rightarrow N}} \Big( N(j)-P(j) \Big) }_{B}
\end{align}
if having $|\Points_{N\rightarrow P}|=(n_P-n_P')$ and $|\Points_{P\rightarrow N}|=0$ does not achieve the optimal solution and instead we need $|\Points_{P\rightarrow N}|=t>0$ and $|\Points_{N\rightarrow P}|=t+(n_P-n_P')$, then it must be the case that $B<0$ but that would imply that we can achieve a better solution by interchanging $|\Points_{P\rightarrow N}|$ many points\footnote{Note that $|\Points_{N\rightarrow P}-\Points^{*}_{N\rightarrow P}|=|\Points_{P\rightarrow N}|$} from the positive set to the negative, this implies that the assignment of $\phi'$ is not optimal which contradicts the inductive hypothesis.

Now since we find the optimal value for each feasible $n_P$, we indeed can find the optimal value by the finding the minimum of those values.
\end{proof}
We note that although the algorithm would find the optimal solution if it exists, the pre-set proportion bounds may lead to an infeasible problem. In that case, our algorithm would terminate without finding a solution. Further, it is not difficult to generalize the algorithm to the case of the $k$-center. 

\begin{theorem} \label{th_nonexact_alg_run_time}
For the two label case $\labs =\{N,P\}$ and $n$ many points. Using algorithm (\ref{alg:alg_non_exact_k_med_means}) we can obtain the optimal cost and solution for all possible distribution of points among the positive $P$ and negative $N$ labels in $O(n(\log{n}+k))$ and the memory required to save the costs and solutions is $O(n\log{n})$.
\end{theorem}
\begin{proof}
Follows the same argument as that of theorem (\ref{th_exact_alg_correct}). There are clearly $n$ many possible solution values and each point may be assigned to $k\leq n$ many possible centers so we need at most $O(n\log{n})$ memory. 
\end{proof}

%\appendix

\end{document}